\documentclass{article} 
\usepackage[preprint]{iclr2027_conference}
\usepackage{times}

\usepackage{amsmath,amsfonts,bm}

\def\eqref#1{equation~\ref{#1}}

\def\1{\bm{1}}

\DeclareMathAlphabet{\mathsfit}{\encodingdefault}{\sfdefault}{m}{sl}
\SetMathAlphabet{\mathsfit}{bold}{\encodingdefault}{\sfdefault}{bx}{n}

\usepackage{enumitem}
\usepackage{amsthm}
\usepackage{graphicx}
\usepackage{subcaption}
\usepackage{amssymb}
\usepackage{booktabs}
\usepackage{titletoc}
\usepackage{hyperref}
\hypersetup{
    colorlinks=true,
    citecolor=purple, 
    linkcolor=purple,  
    urlcolor=cyan,  
    pdfborder={0 0 0} 
}
\usepackage{cleveref}
\usepackage{url}
\usepackage{algorithm}
\usepackage{algpseudocode}
\usepackage[dvipsnames,svgnames,table]{xcolor}

\newtheorem{theorem}{Theorem}
\newtheorem{lemma}{Lemma}

\newtheorem{corollary}{Corollary}
\newtheorem{assumption}{Assumption}
\newtheorem{definition}{Definition}

\definecolor{QiBlueGray}{HTML}{3B6BC4}

\title{Shift-aware score transport for privacy\\ auditing without retraining}

\author{Yaohong Yang\(^{1,2}\), Gauri Pradhan\(^3\), Qi Chen\(^{1,2}\), Samuel Kaski\(^{1,2,4}\), Antti Honkela\(^{3}\)  \\
\(^1\) ELLIS Institute Finland\\
\(^2\)Department of Computer Science, Aalto University, Finland\\
\(^3\)Department of Computer Science, University of Helsinki, Finland \\
\(^4\) Department of Computer Science, University of Manchester, United Kingdom\\
\texttt{\{yaohong.yang, qi.chen.inf, samuel.kaski\}@aalto.fi} \\
\texttt{\{gauri.pradhan, antti.honkela\}@helsinki.fi}
}

\begin{document}

\maketitle

\begin{abstract}
Privacy auditing empirically assesses privacy leakage by using membership inference attacks to establish lower bounds on privacy loss.
We consider an observational auditing setting in which the auditor cannot insert canaries, control data inclusion, or retrain the target model, and must instead rely on fixed sets of known members and non-members. 
This setting is relevant for large deployed models or when retraining is restricted by computational or access constraints. 
In this case, non-members available for auditing may differ in distribution from the training data, causing membership-inference scores to be skewed due to the distribution shift.
We address this problem by estimating the counterfactual in-distribution non-member score distribution using the available out-of-distribution (OOD) non-members. 
Our proposed method learns the relationship between data representations and membership-inference scores from OOD non-members and transports it to the target covariate distribution. 
Because this transport may become unreliable under extrapolation or imperfect conditional invariance, we further introduce a conservative correction calibrated directly to the nominal transported privacy level. 
We construct controlled calibration tasks from the available non-member data to estimate how representation-space displacement inflates the audit, and use this estimate to correct the transport-induced audit error. 
Across multiple datasets with distribution shifts, our method yields substantially more stable privacy estimates while retaining tighter lower bounds than baselines.
\end{abstract}

\section{Introduction} \label{sec:intro}


Membership inference attacks (MIAs) distinguish training-set members from non-members based on a model's behavior, and are widely used to empirically assess privacy leakage of machine learning models \citep{shokri2017membership,yeom2018privacy,carlini2022membership}. 
A fundamental assumption underlying such comparisons is that members and non-members are drawn from same data distributions. 
In practice, however, available non-members may differ systematically from the training distribution \citep{liang2024real,duan2024membership,meeus2025sok}. 
Under such distribution shift, membership-inference scores can capture distributional differences in addition to membership information, effectively confounding membership inference with distribution-shift detection \citep{duan2024membership,das2025blind,zhang2025position}. 
Consequently, the observed membership signal can substantially overestimate privacy leakage \citep{maini2024llm,even2026membership}.

This issue is particularly consequential for privacy auditing, where MIAs are used to establish empirical lower bounds on privacy loss \citep{jagielski2020auditing,tramer2022debugging}. 
Most existing auditing methods avoid this confounding by controlling the training process—for example, by retraining models \citep{jagielski2020auditing,nasr2023tight}, randomizing the inclusion of candidate examples \citep{steinke2023privacy, mahloujifar2025auditing}, or inserting canaries \citep{pillutla2023unleashing,steinke2023privacy}. 
Such intervention may be unavailable for already-deployed or third-party models and can be prohibitively expensive for large models \citep{kazmi2024panoramia,cebere2026privacy}.
Therefore, we consider an observational auditing setting where the model has already been trained and cannot be retrained \citep{kazmi2024panoramia,cebere2026privacy}. 
The auditor only has black-box access to the target model together with fixed sets of examples known to be training-set members and non-members. 
This setting removes the ability to randomize membership, and consequently introduces a challenge that is largely absent when the auditor controls data inclusion: the available non-members might not follow the same distribution as the training data. 
Correcting for this distributional confounding is therefore necessary to obtain a valid lower bound on privacy leakage.

\begin{figure}[t] 
    \begin{center}          
    \centering
        \includegraphics[width=\textwidth]{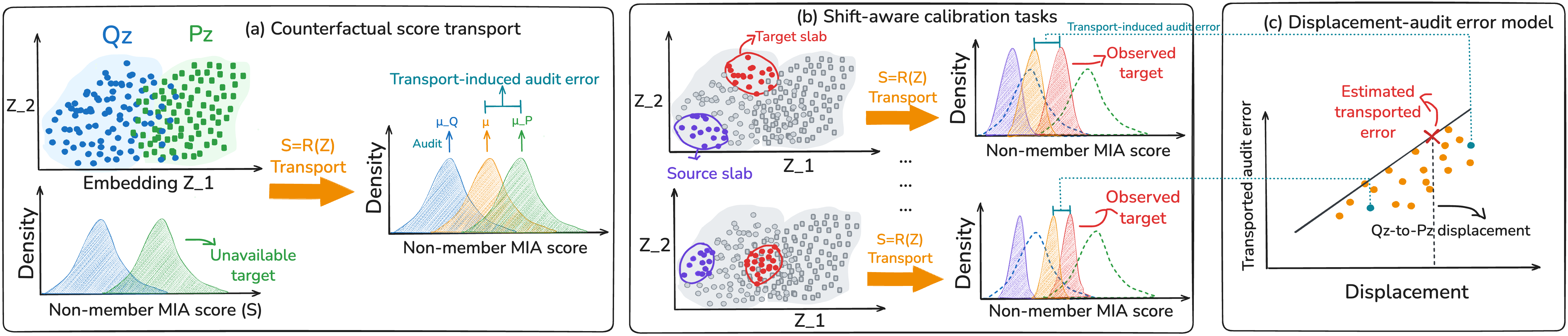}     
    \caption{Overview of our shift-aware privacy auditing framework. 
    (a)\textbf{ Counterfactual score transport.} We learn the relationship \(S=R(Z)\) between data representations \(Z\) and non-member MIA scores \(S\) from the available OOD non-members \(Q_Z\), and transport this relationship to the target representation distribution \(P_Z\) to estimate the unavailable IID non-member score distribution.
    (b)  We construct \textbf{shift-aware calibration tasks} within the OOD non-members by transporting scores from source slabs to target slabs along the \(Q_Z\!\rightarrow\!P_Z\) shift direction. Since the target-slab scores are observed, the resulting transport-induced audit error can be measured.
    (c) The \textbf{displacement–audit error model} measures how the audit error changes with representation displacement, which is evaluated at the \(Q_Z\!\rightarrow\!P_Z\) displacement to estimate the transport-induced error used for correction.
    }
    \label{fig:framework}
    \end{center}
    \vskip -0.2in
\end{figure}

A natural solution is to reconstruct the missing in-distribution non-member score distribution. 
We model how the MIA score varies with data distributions among the available OOD non-members. 
This nominal transport is motivated by a conditional score-invariance assumption: after conditioning on the representation, the non-member score distribution is invariant across domains. 
This allows the conditional score mechanism learned from the reference distribution to be evaluated under the target distribution. 
However, conditional transport alone is insufficient. 
Even when conditional score invariance is valid on the common support, the target distribution may place substantial mass in regions that are poorly represented by the available non-members. 
Predicting scores there necessarily requires extrapolation.
Moreover, the conditional-invariance assumption may be imperfect. 

Our key idea is to correct the transport-induced audit error by constructing controlled calibration tasks within the observed non-member population. 
We partition the OOD non-members into slabs separated along the source-to-target direction, train the same score-transport model on a source-side slab, and evaluate it on a more target-like slab whose true scores remain observable. 
Because both the predicted and observed score samples are available, we can directly measure how much such extrapolation inflates the resulting privacy lower bound. 
Repeating this procedure over multiple controlled displacements yields an empirical relationship between representation-space displacement and transport-induced audit error. 
We then evaluate this error-growth relationship at the actual OOD-to-IID displacement and use the resulting estimate to correct the nominal transported audit. 

As shown in \Cref{fig:framework}, counterfactual score transport (panel (a)) provides a nominal estimate of the missing IID non-member distribution, while the shift-aware calibration (panel (b) and (c)) determines how much of the resulting privacy evidence should be trusted. 
In this work, we use the \(\mu\)-Gaussian differential privacy (\(\mu\)-GDP) formulation \citep{dong2022gaussian} as a framework for quantifying empirical privacy leakage, and report the empirical lower bound on \(\mu\).

Our contributions are threefold:
\begin{itemize}[leftmargin=2em]
    \item We estimate the missing IID non-member score distribution by applying a score model learned from OOD non-members to covariates drawn from the target distribution, under a conditional invariance assumption (\Cref{sec:transport}).
    \item We identify when this assumption may fail (\Cref{sec:diagnostic}) and construct calibration tasks to measure how representation-space displacement affects transport-induced audit error, then use the resulting error-growth estimate to correct the transported audit. (\Cref{sec:correction}).
    \item Across image and language benchmarks with synthetic and natural distribution shifts, we demonstrate that uncorrected observational auditing can substantially overestimate privacy leakage.
    We show that the proposed methods stabilize the audit when transport requires substantial extrapolation or conditional invariance is imperfect, while retaining tighter privacy lower bounds than baselines (\Cref{sec:exps}).

\end{itemize}

\section{Background and related work} \label{sec:background}

\textbf{Privacy auditing.} 
An auditing algorithm takes a learning algorithm \(\mathcal{A}\) and a privacy hypothesis \(\mathcal{H}\), and tests whether the observed behavior of  \(\mathcal{A}\) satisfies \(\mathcal{H}\) \citep{jagielski2020auditing,nasr2023tight}.
Typically, the auditor performs a randomized experiment with known membership labels \(A_i\), applies a MIA to produce guesses \(T_i\in\{-1,0,1\}\) for each example, where \(T_i=0\) denotes abstention, and then evaluates \(\mathcal{H}\) using the attack outcomes and the known memberships.
For an error probability \(p\in [0,1]\), the audit is \(p\)-valid if \(\mathbb{P}(\text{Reject }\mathcal{H}\mid \mathcal{A}\text{ satisfies }\mathcal{H})\leq p\).
Empirical privacy parameters can then be obtained from the strongest privacy hypothesis that the audit can confidently reject.

Early approaches often require retraining models on neighboring datasets \citep{jagielski2020auditing,nasr2021adversary,lu2022general,nasr2023tight,zanella2023bayesian}, while more recent one-run methods reduce this cost by controlling the inclusion of multiple candidate examples within a single training run \citep{steinke2023privacy, mahloujifar2025auditing,xiang2025tight, keinan2026well}. 
However, these approaches still rely on intervention in the training process, which requires the membership indicators \(A_i\in\{-1,1\}\) to be independent.

\textbf{Auditing under distribution shift.}
Distribution shift is known to affect membership inference by changing the score distributions of non-members, potentially overestimate the privacy leakage \citep{yichuan2024assessing,meeus2024inherent,das2025blind,zhao2025unlocking}. 
\cite{kazmi2024panoramia} addresses the lack of in-distribution non-members by training a generative model on known members and using generated samples as non-members, while accounting for distributional mismatch between real and generated data through a baseline classifier.
\cite{cebere2026privacy} instead explicitly corrects the shift through global composition or propensity-based conditional adjustment; its guarantees, however, rely on distributional overlap and accurate propensity estimation, and can become overly conservative under severe shift.
\cite{kalemaj2025observational} evaluates label privacy without intervening on the training dataset by exploiting naturally occurring variation in the data.
Related work has also used conditional modeling, particularly quantile regression, to estimate example-dependent non-member score distributions \citep{bertran2023scalable}. 
\cite{thaker2025membership} study this approach under {class-support shift}, where the available background data omit entire classes or subclasses. 
In contrast, we consider covariate shift, where the member and available non-member populations share the same task and label space but differ in their input distributions.

\textbf{Measuring privacy leakage with \(\mu\)-GDP.}
We quantify empirical privacy leakage using the \(\mu\)-GDP scale, which characterizes distinguishability through the hypothesis-testing trade-off between neighboring datasets \citep{dong2022gaussian}.
\begin{definition}
     An algorithm \(\mathcal{A}\) is \(\mu\)-GDP if for all neighboring datasets \(\mathcal{D}\) and \(\mathcal{D}'\),
    \begin{equation}
        T(\mathcal{A}(\mathcal{D}),\mathcal{A}(\mathcal{D}'))(\alpha)\geq T(\mathcal{N}(0, 1),\mathcal{N}(\mu, 1))(\alpha), \; \forall \alpha \in[0,1],
    \end{equation}
     where \(T(\mathcal{A}(\mathcal{D}),\mathcal{A}(\mathcal{D}'))(\alpha)=\inf\{\beta_\phi:\alpha_\phi\leq\alpha\}\) with Type-I error \(\alpha_\phi\) and Type-II error \(\beta_\phi\) from all rejection rules \(\phi\).
\end{definition}
For privacy audit with \(\mu\)-GDP, the auditor performs membership guesses and compares the observed guessing performance against that permitted by a candidate \(\mu\)-GDP guarantee, yielding a lower confidence bound \(\hat{\mu}\) on the privacy parameter \citep{nasr2023tight,mahloujifar2025auditing}.

\section{Problem formulation} \label{sec:problem}

We consider post-hoc privacy auditing of a fixed target model \(\mathcal M\) without retraining or controlled data inclusion \citep{cebere2026privacy}.
The auditor has only black-box access to the model, meaning it can only observe the model outputs but not its internal representations or intermediate computations.
Let \(s=S(x)\in\mathbb R\) denote a membership-inference score, oriented such that a larger score indicates stronger evidence of membership.
Let \(A\in\{-1,1\}\) denote membership, where \(A=1\) indicates a training-set member. 
Let \(P_X\) denote the training-data distribution and \(Q_X\) an out-of-distribution population from which the available non-members are drawn.
The auditor has access to
\begin{equation}
     \mathcal D_M =\{\mathcal{X}_M,\mathcal{S}_M\}= \{(x_{i,M},s_{i,M})\}_{i=1}^{n_M}, \qquad x_{i,M}\sim P_X, \qquad A_i=1
\end{equation}
consisting of known members and their target-model scores, and
\begin{equation}
    \mathcal D_{N}^Q =\{\mathcal{X}_{N}^Q,\mathcal{S}_{N}^Q\}=\{(x_{i,N}^Q,s_{i,N}^Q)\}_{i=1}^{n_Q}, \qquad x_{i,N}^Q\sim Q_X, \qquad A_i=-1
\end{equation}
consisting of known OOD non-members and their target-model scores.

Given \(n=n_M+n_Q\) auditing samples, suppose the auditor applies a fixed decision rule to predict member, non-member, or abstain, and makes \(c\) membership guesses, of which \(k\) are correct. We denote the resulting lower confidence bound on the \(\mu\)-GDP parameter by
\begin{equation} \label{eq:mu}
    \hat{\mu}(\mathcal S_N^Q,\mathcal S_M)=\mu_L(c,k,n),
\end{equation}
where \(\mu_L(c,k,n)\) is the lower bound obtained from the one-run auditing procedure.
{
When \(P_X=Q_X\), since \(\mathcal{D}_M\) and \(\mathcal{D}_N^Q\) are equal-sized subsets independently sampled from the same distribution, \Cref{eq:mu} remains a valid lower bound in the observational setting, though we cannot perform the sampling prior to model training. 
Because the observed membership labels behave like randomized inclusion bits. 
However, when \(P_X\neq Q_X\), the uncorrected \Cref{eq:mu} is no longer guaranteed to be a valid lower bound, which motivates the correction proposed in our method. 
}

Let \(S_N^P\) denote the IID non-member score \(S(X)\) for \(X\sim P_X\).
The desired reference distribution for privacy auditing is the score distribution of IID non-members 
\begin{equation}
    F_{S_N^P} = \mathbb{P}_{X\sim P_X}(S(X)\le s\mid A=-1),
\end{equation}
which is unavailable at deployment.
Our goal is therefore to correct the privacy lower bound \(\hat{\mu}\) obtained from \((\mathcal{S}_N^Q, \mathcal{S}_M)\) so that it remains valid while being as tight as possible to the bound obtained using the in-distribution non-members \((\mathcal{S}_N^P, \mathcal{S}_M)\).

\begin{algorithm}[t]
\caption{Extrapolation-Calibrated Counterfactual Score Transport}
\label{alg:method}
\begin{algorithmic}[1]
\Require Member data \(\mathcal D_M=(\mathcal{D}_M^{\text{aud}},\mathcal{D}_M^{\text{ref}})\),
OOD non-members \(\mathcal D_N^Q=(D_N^{Q,\text{aud}},D_N^{Q,\text{ref}})\),
pretrained representation \(\phi\),
slab fractions \(\mathcal F\).

\State Compute \(Z_M=\{\phi(x_i^M)\},x_i^M\in \mathcal X_M\) and \(Z_N^Q=\{\phi(x_i^Q)\},x_i^Q\in\mathcal X_N^Q\).

\State \(\tilde{\mathcal S}_N^P\gets\) predictions on \(Z_M^{\mathrm{ref}}\) from the transport learned on \(\mathcal D_N^Q\) (Equation \ref{eq:mean_regression} and \ref{eq:quantile_regression}).

\State \(\hat\mu_{\mathrm{trans}} \gets \hat\mu(\tilde{\mathcal S}_N^{P};\mathcal S_M^{\text{audit}})\).


\For{\(f\in\mathcal F\)} \Comment{\textbf{Shift-aware calibration tasks}}
    \State Take source slab \(\mathcal D_{N,f}^{Q,\mathrm{so}}\) and target slab \(\mathcal D_{N,f}^{Q,\mathrm{ta}}\) from \(\mathcal{D}_{N}^{Q,\text{ref}}\)along the \(Q_Z\rightarrow P_Z\) shift direction.

    \State Compute the representation-space displacement \(\Delta_f \gets\) \Cref{eq:displcement}.

    \State  \(\tilde{\mathcal S}_{f}^{\,\mathrm{pred}}\gets\) predictions on \(Z_{N,f}^{Q,\mathrm{ta}}\) from the transport learned on \(\mathcal D_{N,f}^{Q,\mathrm{so}}\) (Equation \ref{eq:mean_regression} and \ref{eq:quantile_regression}).


    \State  Compute the transport-induced audit error \(e_f  \gets \) \Cref{eq:count-error}.


\EndFor

\State Fit quantile regression on \((\Delta_f, e_f)\): \( (\hat a_\tau,\hat b_\tau) \gets \arg\min_{b\ge0} \sum_{f\in\mathcal F}\rho_\tau(e_f-a-b\Delta_f).\)

\State \(\hat\mu_{\mathrm{corr}} \gets [\hat\mu_{\mathrm{trans}}-\hat r]_+\), \(\hat r \gets [\hat a_\tau+\hat b_\tau \bar d]_+,\; \bar d =\max(d^*,\Delta_{\min}),\;d^*\gets \) \Cref{eq:axis}.

\State \Return Corrected privacy lower bound \(\hat\mu_{\mathrm{corr}}\).
\end{algorithmic}
\end{algorithm}

\section{Methodology} \label{sec:method}

{
Our method consists of three steps, as illustrated in \Cref{fig:framework}. 
First, \Cref{sec:transport} estimates the missing IID non-member score distribution through counterfactual score transport (Figure 1a). 
Second, \Cref{sec:diagnostic} discusses the conditions under which nominal transport may become unreliable. 
Finally, \Cref{sec:correction} constructs controlled within-OOD calibration tasks (Figure 1b) and learns a displacement-dependent correction (Figure 1c). 
\Cref{alg:method} summarizes the complete procedure.
}


We split member data \(\mathcal D_M\) and non-member data \(\mathcal D_N^Q\) into disjoint subsets: reference data \((\mathcal D_M^{\mathrm{ref}}, \mathcal D_N^{Q,\mathrm{ref}})\) and auditing data \((\mathcal D_M^{\mathrm{aud}}, \mathcal D_N^{Q,\mathrm{aud}})\). The auditing data are reserved for the final audit, while the reference data are used to learn how to mitigate the effect of distribution shift.

\subsection{Counterfactual score transport} \label{sec:transport}

 {The intuition is that, for a fixed target model, the MIA score of a non-member is determined by the input and the model's behavior. 
Therefore, when conditioning on the full input (X), the non-member score mechanism is domain-invariant for a deterministic scoring rule.}
However, directly modeling the MIA score as a function of the high-dimensional input \(X\) is difficult, we introduce a fixed pretrained representation \( \phi:{X}\rightarrow{Z}. \)
Let \( P_Z\) and \(Q_Z \) denote the corresponding representation distributions under \(P_X\) and \(Q_X\).


\begin{assumption}[Conditional score invariance] \label{assumption}
For all \(z\) in the common support of \(P_Z\) and \(Q_Z\), \(F_{S_N^P}(s\mid z)=F_{S_N^Q}(s\mid z)\), where
\( F_{S_N^P}(s\mid z)=\mathbb{P}(S_N\leq s\mid Z=z, A=-1, D=P)\) and analogously for \(Q\).
\(D\) denotes the domain.
\end{assumption}

Equivalently, \(S_N\perp D\mid Z, A=-1.\)
Under this assumption \citep{chernozhukov2013inference} and when \(\text{supp}(P_Z)\subseteq \text{supp}(Q_Z)\), the counterfactual target IID non-member score distribution can be written as
\begin{equation} \label{eq:invariance}
     \begin{aligned} 
     F_{S_N^P}(s) = \int \mathbb{P}(S_N\le s\mid Z=z,A=-1,D=P) \,dP_{Z}(z) = \int F_{S_N^Q}(s\mid z) \,dP_{Z}(z).
     \end{aligned} 
\end{equation} 
The OOD non-members provide information about the conditional score mechanism \(F_{S_N^Q}(S\mid Z),\)
while the member inputs provide samples from the target covariate distribution \(P_Z.\)
Thus, we estimate the counterfactual IID non-member MIA score distribution \(\tilde{F}_{S_N^P}\) using the OOD non-member data \(\mathcal D_N^{Q}\) and the target reference representations \(P_Z^{\text{ref}}\) with two transport estimators in the following.

\subsubsection{Linear residual transport} \label{sec:linear_trans}
We model the conditional mean of OOD non-member scores \(m(z)=\mathbb{E}[S_N| Z=z, D=Q],\) and estimate \(m(z)\) from  reference OOD non-members \(\mathcal{D}_N^{Q,\text{ref}}\) by squared-loss regression
\begin{equation} \label{eq:mean_regression}
    \hat m(z)=\arg\min\sum_{(z_i^{Q,\text{ref}},s_i^{Q,\text{ref}})\in\mathcal D_{N}^{Q,\text{ref}}}(s_i^{Q,\text{ref}}-m(z_i^{Q,\text{ref}}))^2.
\end{equation}
For an observed OOD audit non-member \((z_i^{Q,\text{aud}},s_i^{Q,\text{aud}})\in\mathcal{D}_N^{Q,\text{aud}}\), define its residual 
\begin{equation}
    \hat{\epsilon}_i^{Q}=s_i^{Q,\text{aud}}-\hat{m}(z_i^{Q,\text{aud}}).
\end{equation}
We then sample \(B\) independent reference members \(z_{i,1}',\dots,z_{i,B}'\sim P_Z^{\text{ref}} \) from the member covariate pool and transport the score to the same residual under the target covariate:
\begin{equation} \label{eq:linear_transport}
    \tilde{s}_{i,b}^{P,lr}=\hat{m}(z_{i,b}')+\hat{\epsilon}_i^{Q},\; b=1,\dots,B.
\end{equation}
This transformation changes the conditional location of the score according
to the target covariate while preserving the residual.
The linear transport assumes an additive score model \(S_N = m(Z)+\epsilon,\) where the residual distribution is invariant across covariates and domains. 
Consequently, linear transport primarily corrects changes in conditional
location and may fail when distribution shift changes conditional variance,
skewness, or tail behavior.

\subsubsection{Conditional quantile transport} \label{sec:quantile_trans}
To model the full conditional score distribution, we estimate conditional quantile function \citep{koenker2001quantile}
\begin{equation}
    q_\tau(z) = F_{S_N^Q}^{-1}(\tau\mid Z=z), \qquad \tau\in(0,1). 
\end{equation}
For a grid of quantile levels \(\mathcal T\), the quantile models are trained by minimizing pinball loss on \(\mathcal{D}_N^{Q,\text{ref}}\),
\begin{equation} \label{eq:quantile_regression}
    \hat q_\tau(z) = \arg\min \sum_{(z_i^{Q,\text{ref}},s_i^{Q,\text{ref}})\in\mathcal D_N^{Q,\text{ref}}} \rho_\tau \left( s_i^{Q,\text{ref}}-q(z_i^{Q,\text{ref}}) \right),  
\end{equation}
where \( \rho_\tau(u) = u\left(\tau-\mathbf 1\{u<0\}\right). \)

For each audit OOD non-member \((z_i^{Q,\text{aud}},s_i^{Q,\text{aud}})\in\mathcal{D}_N^{Q,\text{aud}}\), we estimate its conditional quantile,
\begin{equation}
    \hat\tau_i = \hat F_{S_N^Q}(s_i^{Q,\text{aud}}\mid z_i^{Q,\text{aud}}). 
\end{equation}
We again sample \(B\) independent members \(z_{i,1}',\dots,z_{i,B}'\sim P_Z^{\text{ref}} \) from the member covariate pool and transport the score to the same conditional percentile under the target covariate: 
\begin{equation} \label{eq:quantile_transport}
    \tilde s_{i,b}^{P,qr} = \hat q_{\hat\tau_i}(z_{i,b}'), \; b=1,\dots,B. 
\end{equation}
The transported sample \(\tilde {\mathcal S}_N^{P,lr} \) and \(\tilde {\mathcal S}_N^{P,qr}\) is obtained by averaging the replicate empirical CDFs, which reduces sample variance without averaging individual records' scores.
Specifically,
\begin{equation}
    \tilde F_{S_{N,b}^P}(s)=\frac{1}{n_Q^{\text{aud}}}\sum_{i=1}^{n_Q^{\text{aud}}} 1\{\tilde{s}_{i,b}^P\leq s\}, \; \tilde F_{S_N^P}(s)=\frac{1}{B}\sum_{b=1}^B \tilde F_{S_{N,b}^P}(s).
\end{equation}
The final \(n_Q^{\mathrm{aud}}\) transported samples  \(\tilde {\mathcal S}_N^{P,lr} \) and \(\tilde {\mathcal S}_N^{P,qr}\) are obtained from \(\tilde F_{S_N^P}\) respectively. 
 {
Specifically, to obtain individual scalar scores, we extract evenly spaced quantiles from this distribution using linear interpolation and assign them to individual non-member records according to their score ranking. 
These individual transported scores are then used to compute the transported audit bound.
}

\subsection{When is nominal transport unreliable?} \label{sec:diagnostic}
Recall that identifying \(F_{S_N^P}\) via \Cref{eq:invariance} requires two conditions:
(a) the conditional invariance assumption (Assumption \ref{assumption}), and
(b) sufficient support overlap, \(\operatorname{supp}(P_Z)\subseteq\operatorname{supp}(Q_Z)\).
Either condition may fail. 
The representation \(Z\) must retain sufficient information about the target-model score for conditional invariance to hold, while support coverage is required to avoid extrapolating the score model beyond the observed source covariates. 
We distinguish the two failure modes:
\begin{enumerate}[nosep,leftmargin=2em]
   \item \textbf{Lack of support overlap.}
    Suppose Assumption \ref{assumption} holds on the overlap region, but
    \(\operatorname{supp}(P_Z)\nsubseteq\operatorname{supp}(Q_Z)\).
    Let
    \(
        \mathcal O=\operatorname{supp}(P_Z)\cap\operatorname{supp}(Q_Z),
        \mathcal O^c=\operatorname{supp}(P_Z)\setminus\mathcal O.
    \)
    Then
    \begin{equation}\label{eq:overlap_intergral}
        F_{S_N^P}(s)
        =
        \int_{\mathcal O} F_{S_N^Q}(s\mid z)\,dP_{Z}(z)
        +
        \int_{\mathcal O^c} F_{S_N^P}(s\mid z)\,dP_{Z}(z).
    \end{equation}
    The second term is not identified from the OOD non-members and requires extrapolation.

    \item \textbf{Violation of conditional invariance.}
    Even within the overlap region, Assumption \ref{assumption} may fail, such that
    \begin{equation}
        F_{S_N^P}(S\mid Z=z)\neq F_{S_N^Q}(S\mid Z=z).
    \end{equation}
    In this case, transporting the conditional score distribution from \(Q_Z\) to \(P_Z\) can be biased even without extrapolation.
\end{enumerate}

\textbf{Diagnostic assessment of conditional transportability.}
In order to assess whether the score distribution learned from OOD non-members transfers to the target domain, we define the domain propensity score \( \pi(z)=\mathbb{P}(D= Q\mid Z=z,\; A=-1),\)
where values near \(0\) and \(1\) indicate IID-like and OOD-like representations, respectively. 
We partition the data into ten propensity bins and, for each bin \(I\), estimate the OOD non-member score CDF: \(F_{Q,I}(s)=\mathbb{P}(S\leq s\mid D=Q,\pi(z)\in I).\)
For an observation with score \(S\), we compute its PIT (probability integral transform) value \(U=F_{Q,I}(S)\). 
If the conditional score distribution is transportable from \(Q\) to \(P\), the PIT values for IID non-members should be approximately uniform, and their ECDF should follow the diagonal.

As shown in \Cref{fig:pit-domainnet} for a distribution-shift benchmark DomainNet \citep{peng2019moment} and across more datasets in \Cref{fig:pit_cifar10c,fig:pit_additional}, the IID non-member curve is either close to the diagonal or lies below it. 
The latter indicates larger IID non-member MIA scores than predicted from OOD non-members, implying that nominal transport tends to under-correct the shift and can yield an anti-conservative privacy lower bound. 
This motivates our calibration tasks to correct the audit error.



\begin{figure}[htbp] 
    \begin{center}          
    \centering
        \includegraphics[width=\textwidth]{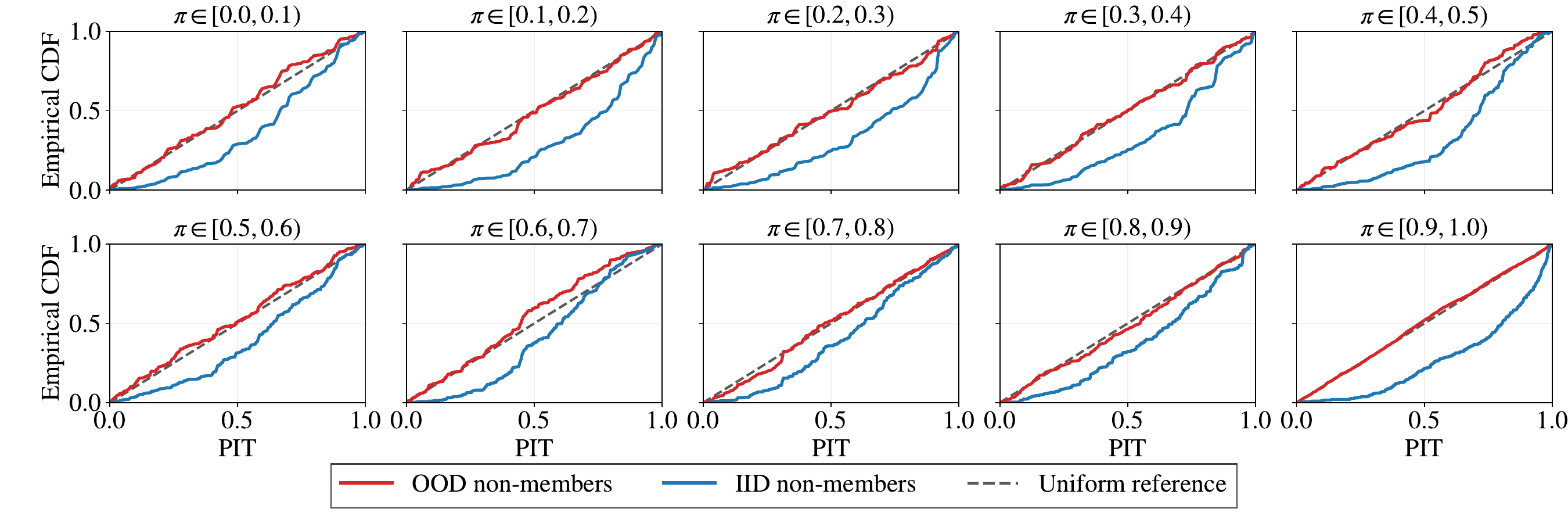}     
    \caption{Diagnostic assessment of conditional transportability.
    The dashed diagonal denotes the uniform reference expected under conditional transportability. 
    The OOD non-members are used to construct the reference score distributions, while IID non-member curves below the diagonal indicate larger MIA scores than predicted by the OOD reference distributions.}
    \label{fig:pit-domainnet}
    \end{center}
    \vskip -0.2in
\end{figure}




\subsection{Shift-aware conservative correction} \label{sec:correction}



 {
Here we present the intuition behind shift-aware calibration. 
Our goal is to estimate how much the nominal score transport may overestimate privacy leakage. 
However, we cannot directly measure this error because the true IID non-member scores are unavailable. 
Instead, we construct surrogate transport tasks within the observed OOD non-members, where the true scores are known. 
The key idea is to create these tasks along the same direction as the actual OOD-to-IID shift. 
Specifically, we first identify the direction from the OOD representation mean to the target representation mean. 
We then project the OOD representations onto this direction and construct source and target slabs at different positions. 
By transporting scores between these slabs, we can measure how the audit error changes with displacement. 
Finally, we estimate the error at the actual OOD-to-IID displacement and use it to correct the nominal privacy lower bound. 
}

 {
Specifically, our calibration procedure consists of three steps: 
(1) \textbf{Source-to-target displacement}: identify the OOD-to-IID shift direction and quantify its magnitude; 
(2) \textbf{Controlled slab extrapolation}: construct source-like and target-like slabs within the OOD data and measure transport-induced audit errors at varying separations; 
(3) \textbf{Error-growth envelope}: fit an envelope to these errors, evaluate it at the OOD-to-IID displacement, and use the estimated audit error to correct the privacy lower bound.
We show the details in the following.
}

\textbf{Source-to-target displacement.}
We reduce the source-to-target mismatch to a scalar. 
Let \(\widehat\Sigma_Q\) be a Ledoit--Wolf shrinkage estimate \citep{ledoit2004well} of the reference covariance \(Z_N^{Q,\text{ref}}\) and \(W = \widehat\Sigma^{-1/2}_Q\) the corresponding whitening map, so that one unit is one reference standard deviation. 
 {We whiten the representations using the OOD covariance to account for differences in feature scale and correlation, so that the shift direction is not dominated simply by high-variance dimensions.
The detailed explanation can be found in \Cref{app:explain_ood-to-iid}.}
With \(\bar z_{N}^Q\) and \(\bar z_{M}^P\) the reference OOD non-member and member means, define the \emph{shift axis} and the induced displacements
\begin{equation}
  \label{eq:axis}
  \mathcal O \;=\; \frac{W(\bar z_{M}^P - \bar z_{N}^Q)}
                {\|W(\bar z_{M}^P - \bar z_{N}^Q)\|},
  \qquad
  d( z) \;=\; \bigl\langle W (z - \bar z_{N}^Q),\; \mathcal O \bigr\rangle,
  \qquad
  d^* \;=\; \bigl\|W(\bar z_{M}^P - \bar z_{N}^Q)\bigr\| .
\end{equation}
\(\mathcal O\) means the shift direction from non-members to members.
Each non-member sample is assigned a signed displacement \(d\), centered to have mean zero, while \(d^*\) denotes the mean displacement of the member pool, representing the distance over which the transport is required to extrapolate.

\textbf{Controlled slab extrapolation.}
\(d(z)\) measures a sample's position along the \(Q_Z\to P_Z\) shift direction, with larger \(d(z)\) indicating a position closer to the target side. 
We therefore sort \(\mathcal{D}_N^{Q,\mathrm{ref}}\) by \(d(z)\). 
For each fraction \(f\in(0,\tfrac12]\), we take the bottom \(f\)-fraction as a source slab \(\mathcal{D}^{Q,\mathrm{so}}_{N,f}\) and the top \(f\)-fraction as a target slab \(\mathcal{D}^{Q,\mathrm{ta}}_{N,f}\). 
This creates a controlled transport task from a source-like region to a more target-like region along the same direction as the actual \(Q_Z\to P_Z\) transport. 
 {Both the source and target slabs consist entirely of observed OOD non-members. 
The target slab simply contains examples located farther along the direction toward the IID representation mean. 
Its scores are therefore observable, allowing us to measure the transport error.}
We quantify the distance of this task by the difference between the mean positions of the two slabs,
\begin{equation}\label{eq:displcement}
  \Delta_f \;=\; \frac{1}{|\mathcal{D}^{\mathrm{Q,ta}}_{N,f}|}\sum_{z_i\in\mathcal{D}^{\mathrm{Q,ta}}_{N,f}}d(z_i)
               - \frac{1}{|\mathcal{D}^{\mathrm{Q,so}}_{N,f}|}\sum_{z_j\in\mathcal{D}^{\mathrm{Q,so}}_{N,f}}d(z_j).
\end{equation}
We refit the same residual and quantile transport (equation \ref{eq:mean_regression} and \ref{eq:quantile_regression}) on $\mathcal{D}^{\mathrm{Q,so}}_{N,f}$ alone, and use it to produce predicted score samples at the covariates of $Z^{\mathrm{Q,ta}}_{N,f}$.
Because \(Z^{\mathrm{Q,ta}}_{N,f}\) is real reference data, its \emph{observed} scores \(\mathcal S^{\mathrm{Q,ta}}_{N,f}\) are known, so we could measure the transport-induced audit error between the predicted and observed samples measures when made to extrapolate a distance \(\Delta_f\).

Let \(\mathcal S_f^{\text{obs}}\) denote the observed score samples and \(\tilde{\mathcal S}_f^{\text{pred}}\) the predicted score samples on \(Z^{\mathrm{Q,ta}}_{N,f}\) learned from \(\mathcal{D}^{\mathrm{Q,so}}_{N,f}\) at separation \( \Delta_f\).
We therefore define the transport-induced audit error:
\begin{equation}
  \label{eq:count-error}
  e_f \;=\; \hat{\mu}(\tilde{\mathcal S}^{\text{pred}}; \mathcal S_M^{\text{ref}}) \;-\; \hat{\mu}(\mathcal S_f^{\text{obs}};\mathcal S_M^{\text{ref}}) .
\end{equation}
\textbf{Fit the error-growth envelope.}
By varying the fraction \(f\), we obtain a collection of pairs \(\{(\Delta_f, e_f)\}\).
The slab separations \(\Delta_f\) are bounded by the reference pool's own spread \([\Delta_{\min},\Delta_{\max}]\), whereas the transport must travel \(d^*\). 
We fit a growth law \(g(\Delta)\) to the pairs \(\{(\Delta_f, e_f)\}\) and read it off at \(d^*\):
\begin{equation}
  \label{eq:growth}
  \hat r \;=\; [\hat g(\bar d)]_+, \qquad \bar d=\max(d^*,\Delta_{\min}), \qquad  e_f= g(\Delta_f),
\end{equation}
where \([c]_+=\max(c,0)\).
In order to obtain an upper estimate of how large the error can become at a given displacement, we use \(\tau=0.95\) to fit an upper conditional quantile of the observed calibration errors.
Let \(\hat\mu_{\text{trans}}=\hat{\mu}(\tilde{\mathcal S}_N^{P};\mathcal S_M^{\text{audit}})\) be the bound the audit would report on the transported samples. 
We define 
\begin{equation}
    \hat{\mu}_{\text{corr}}=[\hat\mu_{\text{trans}}-\hat r]_+,\qquad  e_{\mathrm{real}} = \left[ \hat\mu_{\text{trans}} - \hat\mu_{\text{oracle}} \right]_+. 
\end{equation}
Here \(\hat\mu_{\text{oracle}}=\hat\mu(\mathcal S_N^P;\mathcal S_M)\) is the privacy lower-bound estimator obtained by auditing member scores \(\mathcal S_M\) against the unobserved oracle IID non-member scores \(\mathcal S_N^P\).

For the target representation-space displacement \(d^*\), if the real transport-induced audit error is upper-bounded by the error-growth envelope calibrated from source-domain slab extrapolations \( e_{\mathrm{real}} \le \hat r\), then we have the corrected lower bound being valid \(\hat \mu_{\text{corr}}\leq \hat\mu_{\text{oracle}}\).
\begin{assumption}\label{ass:axis_sufficiency}
    Recall that \(e_f\) is the transport-induced audit error of a calibration-type extrapolation task with displacement \(\Delta_f\).
    Define its conditional upper quantile 
    \begin{equation}
        q_\tau(\Delta_f)=\inf\{t:\mathbb{P}(e_f\le t\mid \xi=\Delta_f)\ge \tau\}.
    \end{equation}
    Now assume 
    \begin{equation}
        \mathbb{P}(e_{\text{real}}\le q_\tau(d^*)\mid d^*)\ge \tau.
    \end{equation}
\end{assumption}
Assumption \ref{ass:axis_sufficiency} means that at the same displacement, the real \(Q_Z\to P_Z\) audit error is no worse, in its upper tail, than the error distribution represented by the controlled within-\(Q_Z\) calibration tasks.
A sufficient condition is one-sided stochastic domination: \(F_{e_{\text{real}}}(t\mid d^*)\ge F_{e_f|\Delta=d^*}(t)\), \(\forall t\), i.e., at the same displacement, the \(Q_Z\to P_Z\) error is stochastically no larger than the within-\(Q_Z\) error.
 {We want to note that Assumption \ref{ass:axis_sufficiency} is a strong assumption and might be difficult to verify directly in practice. 
Our theoretical result therefore provides sufficient conditions under which the proposed correction is conservative relative to the oracle IID audit.
We provide empirical evidence to verify our assumptions in \Cref{app:more_diags}.}
\begin{assumption} \label{ass:monotone}
    The upper error quantile grows monotonically with the displacement:
    \begin{equation}
      \text{If }   \Delta_1\le\Delta_2, \; \text{then } q_{\tau}(\Delta_1)\le q_{\tau}(\Delta_2).
    \end{equation}
\end{assumption}

\begin{assumption}\label{ass:uniform_coverage}
Let \(E_\tau(\Delta)\) be a data-dependent upper confidence envelope constructed from the calibration data. 
Assume that
    \begin{equation}
        \mathbb{P}(E_{\tau}(\Delta)\ge q_{\tau}(\Delta),\; \forall \Delta \in[\Delta_{\min},\Delta_{\max}])\ge 1-\eta.
    \end{equation}
\end{assumption}
\begin{theorem} \label{thm:conservative}
    Assume Assumptions \ref{ass:axis_sufficiency},\ref{ass:monotone},\ref{ass:uniform_coverage} and \(d^*\in[\Delta_{\min},\Delta_{\max}]\).
     We have 
    \begin{equation}
         \mathbb{P}(\hat\mu_{\text{corr}}\le \hat\mu_{\text{oracle}}) \ge \tau -\eta.
    \end{equation}
\end{theorem}
\Cref{thm:conservative} compares the corrected estimate to  \(\hat{\mu}_{\mathrm{oracle}}\), obtained using the unavailable IID
non-member scores. 
It does not directly compare \(\hat{\mu}_{\mathrm{corr}}\) with the true \(\mu\)-GDP parameter of the training algorithm. 
Thus, the theorem characterizes the additional error introduced by distribution shift and score transport, relative to the audit that would have been obtained if IID non-members were available. 
We defer the proof to \Cref{app:proof} in the Appendix.




\section{Experiments} \label{sec:exps}

\textbf{Auditing Setup.}
We use a loss-based MIA and report a conservative empirical lower bound, $\mu_L$, on privacy leakage.
To account for finite-sample uncertainty in the data-dependent correction, we use 200 bootstrap resamples of the reference dataset.
We allocate the error probability of \(0.025\) to the privacy audit and \(0.025\) to this bootstrap uncertainty quantification.
We repeat each scenario for 5 seeds and plot the mean and standard deviation.
The configuration details can be found in \Cref{app:exp_conf}.

\textbf{Datasets.}
For \textbf{image classification}, we train target models on CIFAR-10 \citep{krizhevsky2009learning} and consider two forms of distribution shift.  
First, we use the \emph{brightness} corruption from CIFAR-10-C \citep{hendrycks2019benchmarking} at all five severity levels, from mild to severe. 
Second, we use CINIC-10 \citep{darlow2018cinic} to introduce a natural source-domain shift, which combining CIFAR-10 (22\%) with downsampled ImageNet images (78\%).
We additionally evaluate on DomainNet \citep{peng2019moment}, a more challenging domain-shift benchmark containing the same 345 object classes across multiple visual domains. 
We train the target model on the \emph{Real} domain.
Images from the \emph{Painting} domain serve as the OOD non-members.
For \textbf{language modeling}, we fine-tune a GPT-2 model on arXiv papers from February 2023 and use arXiv papers from February 2026 as OOD non-members, inducing a
natural temporal distribution shift.

\textbf{Baselines.}
We compare our method against four baselines:
(1) \textbf{No correction}, which directly applies one-run \(f\)-DP auditing to the available OOD non-members without accounting for distribution shift \citep{mahloujifar2025auditing};
(2) \textbf{Oracle IID Non-member}, which applies the same one-run \(f\)-DP audit using true IID non-members and serves as the oracle reference;
(3) \textbf{Global correction} \citep{cebere2026privacy}, which adjusts the audit using a dataset-level estimate of the distribution shift between members and non-members; and
(4) \textbf{Conditional correction} \citep{cebere2026privacy}, which uses propensity scores to account for the shift conditionally on observed covariates.

\begin{figure}[htbp] 
    \begin{center}          
    \centering
        \includegraphics[width=0.9\textwidth]{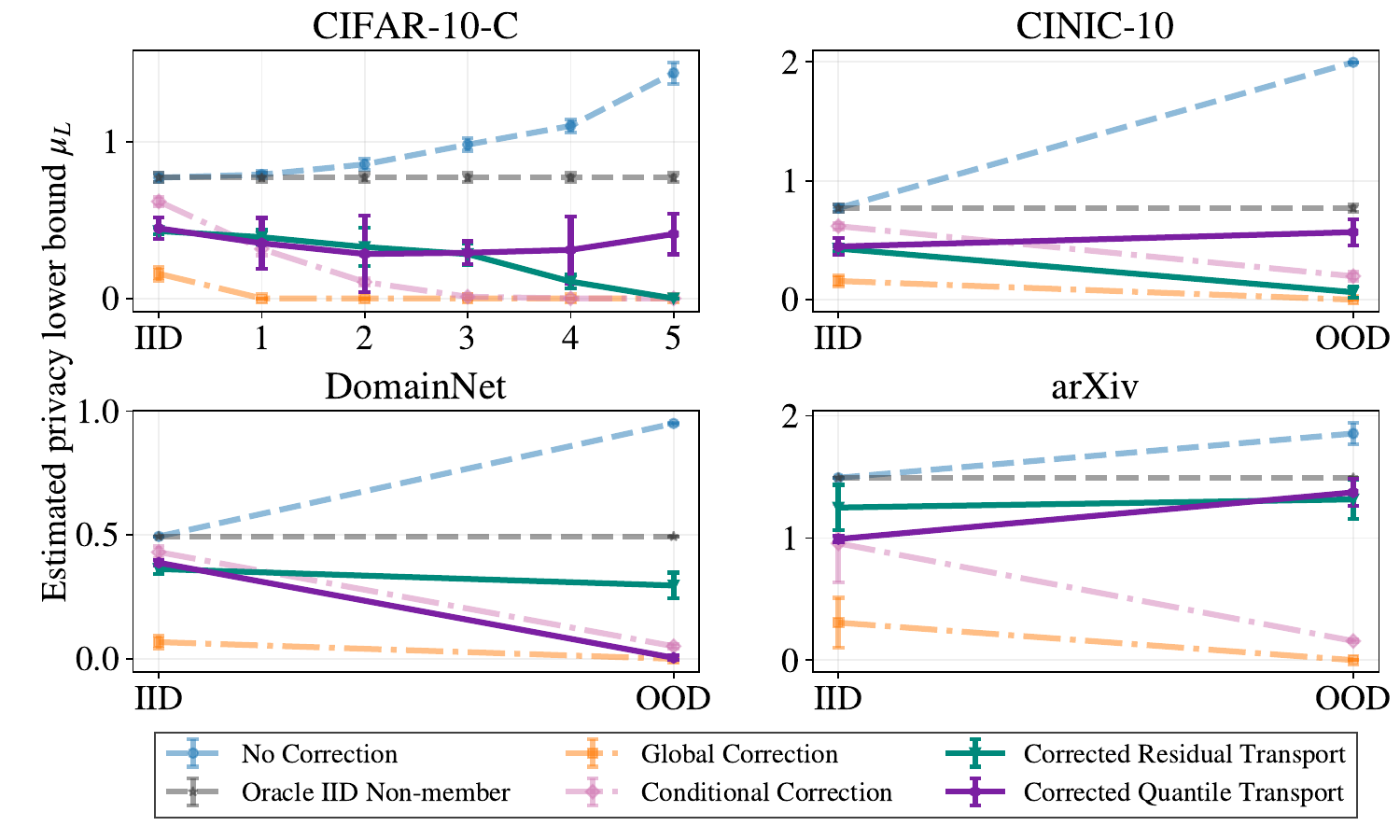}     
    \caption{Comparison of estimated lower bounds across distribution shifts and datasets. At least one of our two transport variants
remains stable and produce bounds that stay close to the oracle.}
    \label{fig:estimated_mu}
    \end{center}
    \vskip -0.1in
\end{figure}

\textbf{Results.}
As shown in \Cref{fig:estimated_mu}, we take the audit with IID non-members as the oracle. 
Without correction, the estimated \(\mu\) increases substantially under distribution shift, indicating that the audit is confounded by distribution shift rather than reflecting privacy leakage alone. 
The global and conditional correction baselines \citep{cebere2026privacy} become increasingly conservative as the shift grows, often yielding loose or near-zero bounds. 
In contrast, at least one of our two transport variants remains stable across datasets and shift scenarios and produces bounds that stay close to the oracle IID audit.
 {
Notice that in IID setting, all correction methods (ours and baselines) exhibit a loss of privacy evidence. 
This is due to the correction being conservative.
Since our methods account for potential transport-induced errors, it can reduce the estimated privacy lower bound even in the IID setting, where such a correction is unnecessary.
Nevertheless, \Cref{fig:estimated_mu} demonstrates that our methods remain robust under distribution shifts, which is our primary objective. 
This reflects a trade-off between conservativeness and auditing effectiveness. 
When the data distributions are known to be IID, our method may not be the optimal choice, as existing auditing methods specifically designed for the IID setting may achieve tighter privacy lower bounds.
}

Under the strongest CIFAR-10-C corruption (level 5), the uncorrected lower bound increases from \(\mu=0.76\) in the IID setting to \(1.51\), while corrected quantile transport yields \(0.42\). 
On CINIC-10, the uncorrected bound similarly increases to \(2.00\) under distribution shift, whereas corrected quantile transport remains at \(0.59\). 
On DomainNet, corrected residual transport yields \(0.31\) in the OOD setting, compared with an IID bound of \(0.49\). 
Finally, on arXiv, both corrected residual and quantile transport remain close to the IID bound of \(1.49\), yielding \(1.37\) and \(1.39\), respectively.

We further visualize the fitted relationship between representation-space displacement \({\Delta}_f\) and the transport-induced audit error \( e_f\) in \Cref{fig:quantile_fit}, together with the target displacement and its predicted and realized errors. 
Across all evaluated settings, the predicted error upper-bounds the realized error, supporting the conservativeness of the resulting corrected privacy lower bounds.

 {
More results can be found in \Cref{app:more_results}. 
Specifically, \Cref{app:data_shift} characterizes representation distribution shifts across benchmarks.
\Cref{app:bonferroni} reports the results with Bonferroni correction. 
\Cref{app:running_time} lists the running time of our methods across domains.
\Cref{app:criteo} presents experiments on click-through rate prediction. 
} 
Additional ablations in \Cref{app:ablation} assess the role of the conservative correction in preventing overestimation after score transport, and of the number of fractions used in calibration, to demonstrate the robustness of our method to this hyperparameter.

\begin{figure}[t] 
    \begin{center}          
    \centering
        \includegraphics[width=\textwidth]{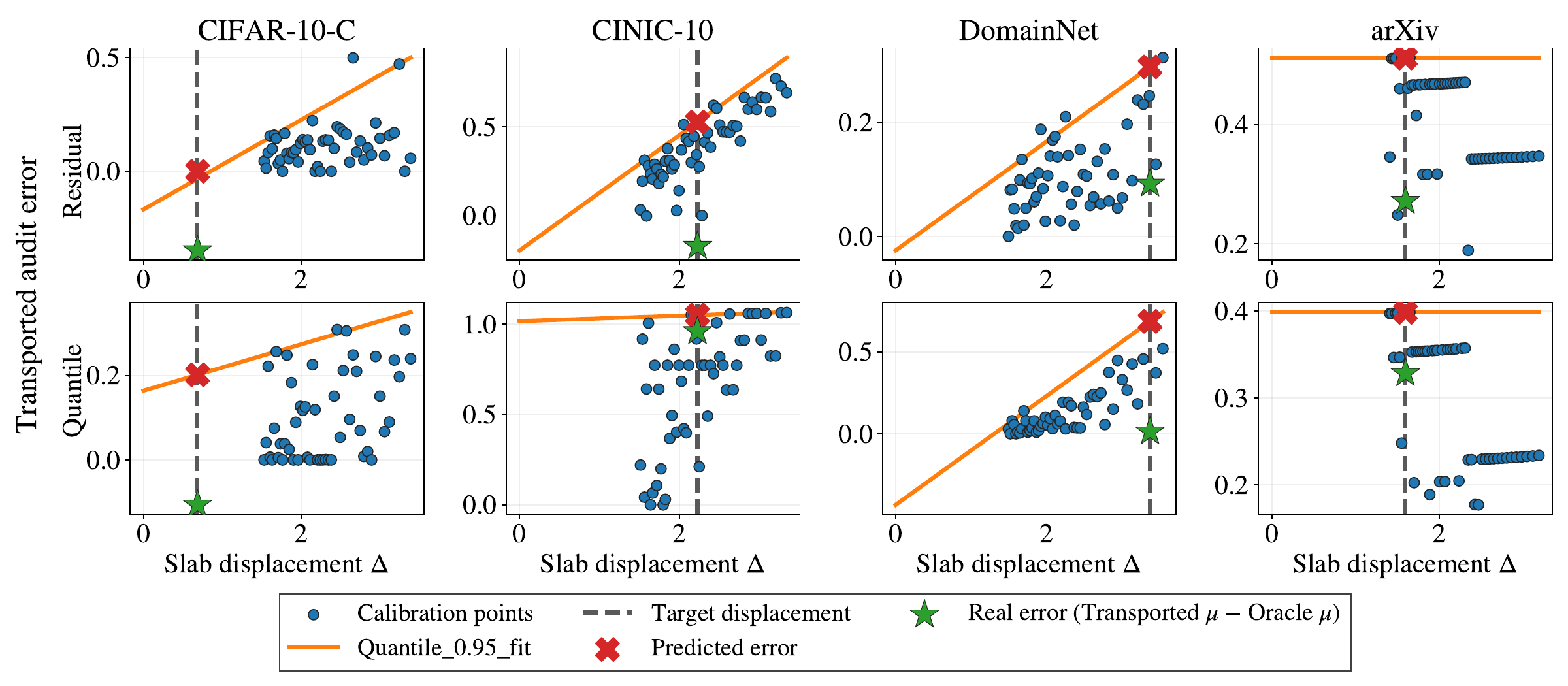}     
    \caption{{Calibration of audit-error growth versus representation-space displacement across datasets (severity level = 1 on CIFAR-10-C).} The fitted upper envelope predicts the target audit error (red cross) at the target displacement and conservatively exceeds the realized audit error (green star).}
    \label{fig:quantile_fit}
    \end{center}
    \vskip -0.2in
\end{figure}

\section{Discussion} \label{sec:discussion}
We study observational privacy auditing in settings where the auditor cannot retrain the target model. 
We estimate the counterfactual IID non-member score distribution by transporting the relationship between pretrained representations and MIA scores learned from OOD non-members. 
To mitigate the transport-induced audit error, we further construct calibration tasks to estimate the error and correct the resulting audit. 
Across multiple distribution shifts, at least one of our two transport variants provides tighter privacy lower bounds than existing baselines. 
In practice, we therefore report the larger corrected lower bound produced by the two variants.
 {A key component of our method is the construction of calibration tasks, where we build slabs along the OOD-to-IID shift direction characterized by the scalar mean displacement. 
Although this scalar displacement does not fully capture the complexity of high-dimensional distribution shifts, it can still serve as a useful empirical predictor of transport-induced auditing error, even when the underlying shifts involve changes beyond the mean. 
Still, we acknowledge that this characterization may be insufficient for certain types of distribution shifts, particularly severe or complex ones.
Another limitation is that our theoretical guarantee relies on strong assumptions that may be difficult to verify in practice. 
}
Nevertheless, 
establishing formal validity guarantees under realistically verifiable assumptions remains challenging for existing observational auditing methods under distribution shift, and addressing this challenge is an important direction for future work.

\section*{Generative AI Use}
We used large language models (LLMs), including OpenAI’s ChatGPT and Anthropic’s Claude, to assist with language polishing, literature discovery, code development, and figure generation. 
Specifically, LLMs were used to improve grammar, phrasing, and notation consistency; identify potentially relevant related work; and assist with writing and refactoring code for data processing, experiment execution, and visualization. 
All suggested references were checked against the original sources, and all AI-assisted code was reviewed, tested, and modified by the authors to ensure correctness and reproducibility. 
The research ideas, experimental design, hypotheses, and scientific claims of the paper were developed and written by the authors. 
LLMs were used only as supporting tools rather than to generate the core scientific contributions. 
We reviewed and finalized all AI-assisted content and take full responsibility for the final submission.

\section*{Reproducibility Statement}
We provide pseudocode for the complete method and describe the score transport, calibration, and auditing procedures in the main text. 
The appendix contains dataset preprocessing details, model and training configurations, MIA scores, representation models, hyperparameters, bootstrap settings, and additional ablations and diagnostics. 
All datasets used in our experiments are publicly available. 
We plan to release the code upon publication.

\section*{Acknowledgments}
This work was supported by the Research Council of Finland (Finnish Center for Artificial Intelligence, FCAI, Grant 356499 and Grant 359111), UKRI Turing AI World-Leading Researcher Fellowship, EP/W002973/1, the Strategic Research Council at the Research Council of Finland (Grant 358247), the European Union (Project 101070617) as well as the
European Large Open Multi-Modal Foundation Models For Robust Generalization On Arbitrary Data Streams (ELLIOT) (Grant 101214398).
Views and opinions expressed are however those of the author(s) only and do not necessarily reflect those of the European Union or the European Commission. 
Neither the European Union nor the granting authority can be held responsible for them. 
This work has been performed using resources provided by the Aalto Science-IT Project from Computer Science IT.
The authors acknowledge the research environment provided by ELLIS Institute Finland.

\bibliography{iclr2027_conference}
\bibliographystyle{iclr2027_conference}

\newpage
\appendix
\startcontents[appendices]
\printcontents[appendices]{}{1}{
    \section*{Appendix Contents}
    \setcounter{tocdepth}{2}
}

\section{Theoretical results} \label{app:proof}

\paragraph{Proof of Theorem~\ref{thm:conservative}.}
Let \(\mathcal{E}= \left\{E_\tau(\Delta) \ge q_\tau(\Delta),\; \ \forall \Delta\in[\Delta_{\min},\Delta_{\max}]\right\}.\)
By Assumption~\ref{ass:uniform_coverage},
\begin{equation}
    \mathbb{P}(\mathcal{E}) \ge 1-\eta.
\end{equation}
Recall that \( \bar d = \max(d^\ast,\Delta_{\min}), \hat r = E_\tau(\bar d).\)
Since $d^\ast\in[\Delta_{\min},\Delta_{\max}]$, we have
$\bar d=d^\ast$. More generally, by the monotonicity in Assumption~\ref{ass:monotone},
\begin{equation}
    q_\tau(d^\ast)\le q_\tau(\bar d).
\end{equation}
On the event \(\mathcal{E}\),
\begin{equation}
    q_\tau(d^\ast)\le q_\tau(\bar d)\le E_\tau(\bar d)= \hat r.
\end{equation}
By Assumption~\ref{ass:axis_sufficiency},
\begin{equation}
    \mathbb{P}\left(e_{\mathrm{real}}\le q_\tau(d^\ast)\right)\ge \tau.
\end{equation}
Therefore, by the union bound,
\begin{equation}
    \mathbb{P}\left(e_{\mathrm{real}}\le \hat r\right)\ge\mathbb{P}\left(\{e_{\mathrm{real}}\le q_\tau(d^\ast)\}\cap \mathcal{E}\right)\ge \tau-\eta.
\end{equation}

It remains to show that \(e_{\mathrm{real}}\le\hat r\) implies \(\hat\mu_{\mathrm{corr}}\le\hat\mu_{\mathrm{oracle}}\).
Recall that
\begin{equation}
    e_{\mathrm{real}}=[\hat\mu_{\mathrm{trans}}-\hat\mu_{\mathrm{oracle}}]_+, \qquad  \hat\mu_{\mathrm{corr}}=[\hat\mu_{\mathrm{trans}}-\hat r]_+.
\end{equation}
On the event \(e_{\mathrm{real}}\le\hat r\),
\begin{equation}
    \hat\mu_{\mathrm{corr}}=[\hat\mu_{\mathrm{trans}}-\hat r]_+\le[\hat\mu_{\mathrm{trans}}-e_{\mathrm{real}}]_+\le\hat\mu_{\mathrm{oracle}},
\end{equation}
where the last inequality follows directly from the definition of
\(e_{\mathrm{real}}\).

Consequently,
\begin{equation}
    \mathbb{P}\left(\hat\mu_{\mathrm{corr}}\le\hat\mu_{\mathrm{oracle}}\right)\ge\mathbb{P}\left(e_{\mathrm{real}}\le\hat r\right)\ge\tau-\eta.
\end{equation}
\hfill$\square$

\begin{corollary} \label{app:cor}
Under Assumption~\ref{ass:monotone}, if \(d^*<\Delta_{\min}\), then \(    q_\tau(d^*) \le q_\tau(\Delta_{\min}).\)
Therefore, evaluating the error envelope at \( \bar d=\max(d^*,\Delta_{\min})=\Delta_{\min}\) remains conservative for the target displacement \(d^*\) when \(d^*<\Delta_{\min}\).
In contrast, when \(d^*>\Delta_{\max}\), the target transport lies beyond the range of displacements represented by the calibration tasks.
Assumption~\ref{ass:monotone} only implies \(q_\tau(d^*) \ge q_\tau(\Delta_{\max}),\) 
so the error at \(d^*\) cannot in general be upper-bounded by the calibrated envelope without an additional assumption on how \(q_\tau(\Delta)\) extrapolates beyond \(\Delta_{\max}\).
Therefore, we do not claim the conservativeness guarantee of Theorem~\ref{thm:conservative} when \(d^*>\Delta_{\max}\).
\end{corollary}

\section{Experiments configuration}  \label{app:exp_conf}

\paragraph{Membership inference attacks algorithm.}
For each datapoint \(X_i\), the attacker first computes an MIA score \(s_i\in\mathbb{R}\), where a larger score indicates greater confidence that \(X_i\) is a member (i.e., \(A_i=1\)). 
The attacker then predicts the \(c_+\) datapoints with the largest scores as members and the \(c_-\) datapoints with the smallest scores as non-members. 
All remaining datapoints are treated as abstentions, yielding a prediction vector \(T\in\{-1,0,1\}^n\).

\paragraph{Pretrained representations.}
For the image-classification experiments, we use a frozen DINOv2-B/14 model \citep{oquab2023dinov2} as the feature extractor and take its 768-dimensional class-token representation. 
For the language-modeling experiment, we use a frozen, pretrained GPT-2 model \citep{radford2019language} that is separate from the fine-tuned target model. 
We divide each paper into 64-token chunks, mean-pool the token representations within each of the first 10 chunks, and then average across chunks to obtain a 768-dimensional representation. 
The representation models are not fine-tuned during auditing.

\paragraph{Audit and reference samples.}
We construct balanced and mutually disjoint audit and reference sets. 
For CIFAR-10-C and CINIC-10, each audit set contains 5,000 target-model members and 5,000 non-members, while the reference set contains another 5,000 members and 5,000 non-members. 
For DomainNet, the corresponding sets each contain 20,000 members and 20,000 non-members. 
For arXiv, they each contain 3,000 members and 3,000 non-members. 
Thus, the total audit/reference set sizes are respectively 10,000/10,000 for the CIFAR-10 experiments, 40,000/40,000 for DomainNet, and 6,000/6,000 for each arXiv audit case. 
 {The CINIC-10 test set consists of 10,000 images from the CIFAR-10 training set, 10,000 images from the CIFAR-10 test set, and 70,000 images from ImageNet. 
We sample 10,000 OOD non-members exclusively from the CIFAR-10 test and ImageNet subsets, excluding all images originating from the CIFAR-10 training set. 
This ensures that the sampled non-members do not overlap with the target model's training data. }
Reference members are drawn from the target-model training distribution but are disjoint from the audited members;
reference non-members are disjoint from the audited non-members and are drawn from the same non-member domain.

For the \(\mu\)-GDP hypothesis test, we sweep over the number \(c\) of highest-confidence membership guesses. 
We use
\begin{equation}
\begin{aligned}
\text{CIFAR-10-C and CINIC-10:} &\quad c\in\{1000,2000,3000,4000\},\\
\text{DomainNet:} &\quad c\in\{4000,8000,12000,16000\},\\
\text{arXiv:} &\quad c\in\{1000,2000,2500,2800\}.
\end{aligned}
\end{equation}
For every method and audit case, we evaluate all configured guess counts and
report the count that yields the largest valid lower bound on \(\mu\).

\paragraph{Transport models.}
The linear and quantile transports use the same residual MLP backbone. 
The network first maps the input representation to 256 dimensions, applies two width-256 residual blocks, and then maps it to a 128-dimensional hidden representation. 
We use GELU activations, layer normalization, and dropout with probability \(0.1\). 
For image experiments, we concatenate a one-hot class encoding to the DINOv2 representation and perform transport within the same class; the arXiv model is fitted without class conditioning.

The linear transport has a scalar output and is trained using mean-squared error to estimate the conditional mean attack score \(m_\theta(z)\). 
Given an observed non-member score \(s\), representation \(z\), and a reference-member representation \(z'\), its transported score is
\begin{equation}
      \tilde{s}= m_\theta(z') + s - m_\theta(z),
\end{equation}
which preserves the residual around the estimated conditional mean. 
The quantile transport predicts the conditional score quantiles \(Q_\theta(\tau\mid z)\) at
\begin{equation}
    \mathcal{T} =\{0.005,0.01,0.02,0.05,0.1,0.25,0.5,0.75,
  0.9,0.95,0.98,0.99,0.995\}
\end{equation}
using the pinball loss. 
It first estimates the conditional rank of \(s\) at \(z\) and then evaluates the same rank at \(z'\), thereby preserving the sample's conditional quantile under transport.
We use 100 independent member-anchor replicates for both linear and quantile transport, which mean \(B=100\) in both \Cref{eq:linear_transport} and \Cref{eq:quantile_transport}.

Both networks are trained only on the reference data using Adam, with learning rate \(10^{-3}\), weight decay \(10^{-4}\), and at most 250 epochs. 
We reserve 20\% of the training data for validation and apply early stopping with patience 20 and minimum improvement \(10^{-4}\).

\paragraph{Shift-aware calibration.}
To calibrate the support-aware shrinkage correction, we use 50 slab fractions uniformly spaced from \(0.5\) to \(0.1\), inclusive:
\begin{equation}
       f_j = 0.5-\frac{0.4j}{49}, \qquad j=0,\ldots,49.
\end{equation}
For each fraction \(f_j\), the bottom \(f_j\) fraction of the reference non-members along the estimated shift direction is used as the source slab, and the top \(f_j\) fraction is used as the target slab.
Consequently, smaller fractions correspond to more widely separated source and target slabs. 
The resulting displacement-error pairs are used to estimate the transport error at the observed audit displacement.

\paragraph{Bootstrap Algorithm} 
We adopt the same bootstrap algorithm as in \citep{cebere2026privacy} and list it here in \Cref{alg:bootstrap} for reference. 
We resample the reference members and non-members and rerun \Cref{alg:method} for each bootstrap replicate, including score transport, reconstruction of the calibration tasks, and estimation of the transport-induced audit error. 
This produces a corrected privacy lower bound for each replicate, from which we report a lower bootstrap quantile.

\begin{algorithm}[t]
\caption{Bootstrap upper confidence estimate}
\label{alg:bootstrap}
\begin{algorithmic}[1]
\Statex \textbf{Inputs:} Confidence probability \(1-p'\), number of bootstraps
\(K\in\mathbb{N}^*\),
reference member and non-member sample \(\mathcal D_M^{\text{ref}}, \mathcal{D}_N^{Q,\text{ref}}\).

\For{\(k \in \{1,\ldots,K\}\)}
    \State \(\mathcal D_M^k, \mathcal{D}_N^k \gets\) sample \(n_M^{\text{ref}},n_{N}^{Q,\text{ref}}\) points from
    \(\mathcal D_M^{\text{ref}}, \mathcal{D}_N^{Q,\text{ref}}\) with replacement.
    \State \(\hat \mu_{\text{corr}}^k\gets\) Run \Cref{alg:method} on \(\mathcal D_M^k, \mathcal{D}_N^k\).
\EndFor

\State \(\hat\mu_{\text{corr}}\gets\) Run \Cref{alg:method} on \(\mathcal D_M^{\text{ref}}, \mathcal{D}_N^{Q,\text{ref}}\).

\State
\(
    \tilde{\mu}_{\text{corr}}^k
    \gets
    \hat{\mu}^k_{\text{corr}}
    + \hat{\mu}_{\text{corr}}
    - \operatorname{median}
    \left(
    (\hat{\mu}^k_{\text{corr}})_{k\in[K]}
    \right).
\)

\State \Return
\(
    \operatorname{Quantile}
    \left(
    (\tilde{\mu}_{\text{corr}}^k)_{k\in[K]},\,p'
    \right).
\)
\end{algorithmic}
\end{algorithm}

\section{Additional experiments results} \label{app:more_results}

\subsection{ {Characterizing representation distribution shifts across benchmarks}} \label{app:data_shift}
 {
We characterize distribution shift by comparing the first two moments of frozen representations of IID and OOD nonmembers. 
Recall that \(P\) denote the IID distribution and \(Q\) the OOD distribution. 
For each setting, we apply the same frozen encoder \(\phi:\mathcal X\rightarrow\mathbb R^{768}\) to both populations, using DINOv2 for images and pretrained GPT-2 representations for arXiv papers. 
}

 {
Given evaluation representations \(\{z_i^P\}_{i=1}^{n_P}\) and \(\{z_j^Q\}_{j=1}^{n_Q}\), we estimate their means and covariances as
\begin{equation}
    \widehat\mu_D=\frac{1}{n_D}\sum_{i=1}^{n_D}z_i^D,
\qquad
\widehat\Sigma_D=
\frac{1}{n_D-1}\sum_{i=1}^{n_D}
(z_i^D-\widehat\mu_D)(z_i^D-\widehat\mu_D)^\top,
\quad D\in\{P,Q\}.
\end{equation}
Each covariance is centered around its own domain mean, so changes in location are separated from changes in covariance.
}

 {
The left panel in \Cref{fig:shift} reports the normalized mean shift
\begin{equation}
    d_\mu(P,Q)=
\frac{\|\widehat\mu_Q-\widehat\mu_P\|_2}
{\sqrt{\operatorname{tr}(\widehat\Sigma_P)}}.
\end{equation}
The numerator measures the distance between representation centers, while the denominator measures the overall IID representation spread. Thus, \(d_\mu\) expresses the mean displacement relative to the IID population’s root-mean-square spread.
}

 {
The right panel in \Cref{fig:shift} reports the relative covariance shift
\begin{equation}
    d_\Sigma(P,Q)=
\frac{\|\widehat\Sigma_Q-\widehat\Sigma_P\|_F}
{\|\widehat\Sigma_P\|_F},
\qquad
\|A\|_F=\sqrt{\sum_{k,\ell}A_{k\ell}^{\,2}}.
\end{equation}
This quantity captures changes in both coordinate variances and cross-coordinate covariances, including changes in covariance orientation that eigenvalue comparisons alone may miss. 
}

 {
\Cref{fig:shift} shows empirical differences in both representation means and covariances across all settings. 
DomainNet exhibits the largest gaps under these measures, while CIFAR-10-C brightness produces progressively larger gaps as severity increases. 
CINIC-10 also exhibits a substantial mean and covariance difference.
These results demonstrate that all datasets exhibit differences in both the mean and covariance structure of their representation distributions.
}
\begin{figure}[htbp] 
    \begin{center}          
    \centering
        \includegraphics[width=\textwidth]{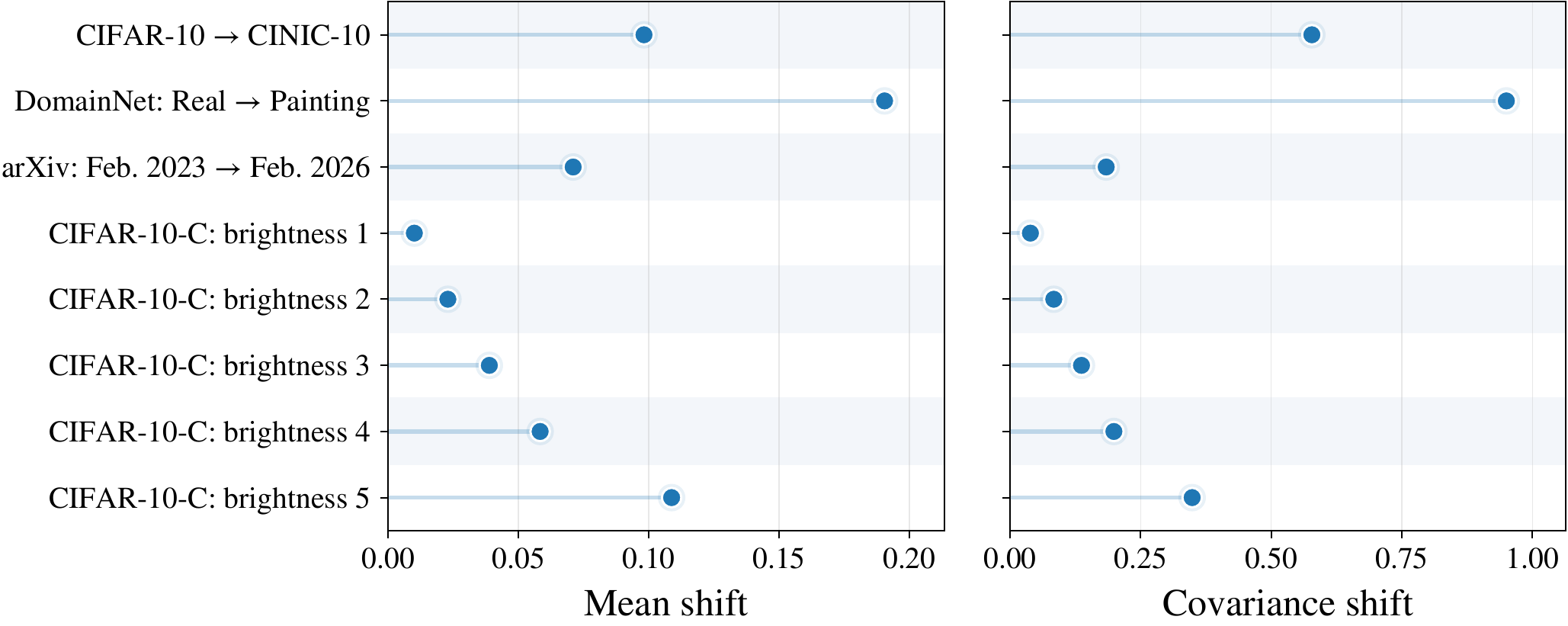}     
    \caption{ {\textbf{Mean and covariance shifts between IID and OOD nonmember representations.} 
    The left panel shows the distance between empirical representation means, normalized by the IID root-total variance. 
    The right panel shows the Frobenius distance between empirical covariance matrices, normalized by the IID covariance norm. 
    }}
    \label{fig:shift} 
    \end{center}
    \vskip -0.2in
\end{figure}

\subsection{ {Bonferroni correction}} \label{app:bonferroni}
 {
As described in \Cref{app:exp_conf}, we sweep across different possible values of guess counts \(c\) and report the largest obtained lower bound. 
Though this is a common practice and has also been commonly adopted in prior privacy-auditing work \citep{steinke2023privacy,mahloujifar2025auditing}, we acknowledge that without a correction, the selected maximum can affect the nominal confidence level \citep{dorseuil2026reliable}.
Standard corrections such as Bonferroni correction \citep{dunn1961multiple} can be applied to maintain the nominal confidence level.}
 {
Let \(\mathcal{C}\) denote the set of candidate numbers of guesses, and let \(\mu_c(\alpha)\) be a lower confidence bound on the true privacy parameter \(\mu\) with confidence level \(1-\alpha\), for a fixed \(c\). The Bonferroni-corrected lower bound is
\begin{equation}
    \mu_{\mathrm{Bonf}}=\max_{c\in\mathcal{C}}\mu_c\left(\frac{\alpha}{|\mathcal{C}|}\right).
\end{equation}
By the union bound,
\begin{equation}
    \Pr(\mu_{\mathrm{Bonf}}>\mu)
\leq\sum_{c\in\mathcal{C}}
\Pr\left(\mu_c\left(\frac{\alpha}{|\mathcal{C}|}\right)>\mu\right)
\leq\alpha.
\end{equation}
Thus, the selected lower bound maintains a confidence level of at least \(1-\alpha\), regardless of the dependence between the individual estimates. 
}

 {
The results can be found in \Cref{fig:bonferroni}.
Since we consider only 4 guess counts, applying the Bonferroni correction results in slightly more conservative privacy lower bounds across all methods. 
The largest relative decrease is 15.7\%, observed for corrected quantile transport on CINIC-10, where the lower bound decreases from 0.5262 to 0.4434. 
Nevertheless, the correction does not change the relative ranking of the methods or our overall conclusions. 
Our methods continue to provide robust privacy lower bounds and tighter than baselines under OOD shifts. 
}
\begin{figure}[htbp] 
    \begin{center}          
    \centering
        \includegraphics[width=\textwidth]{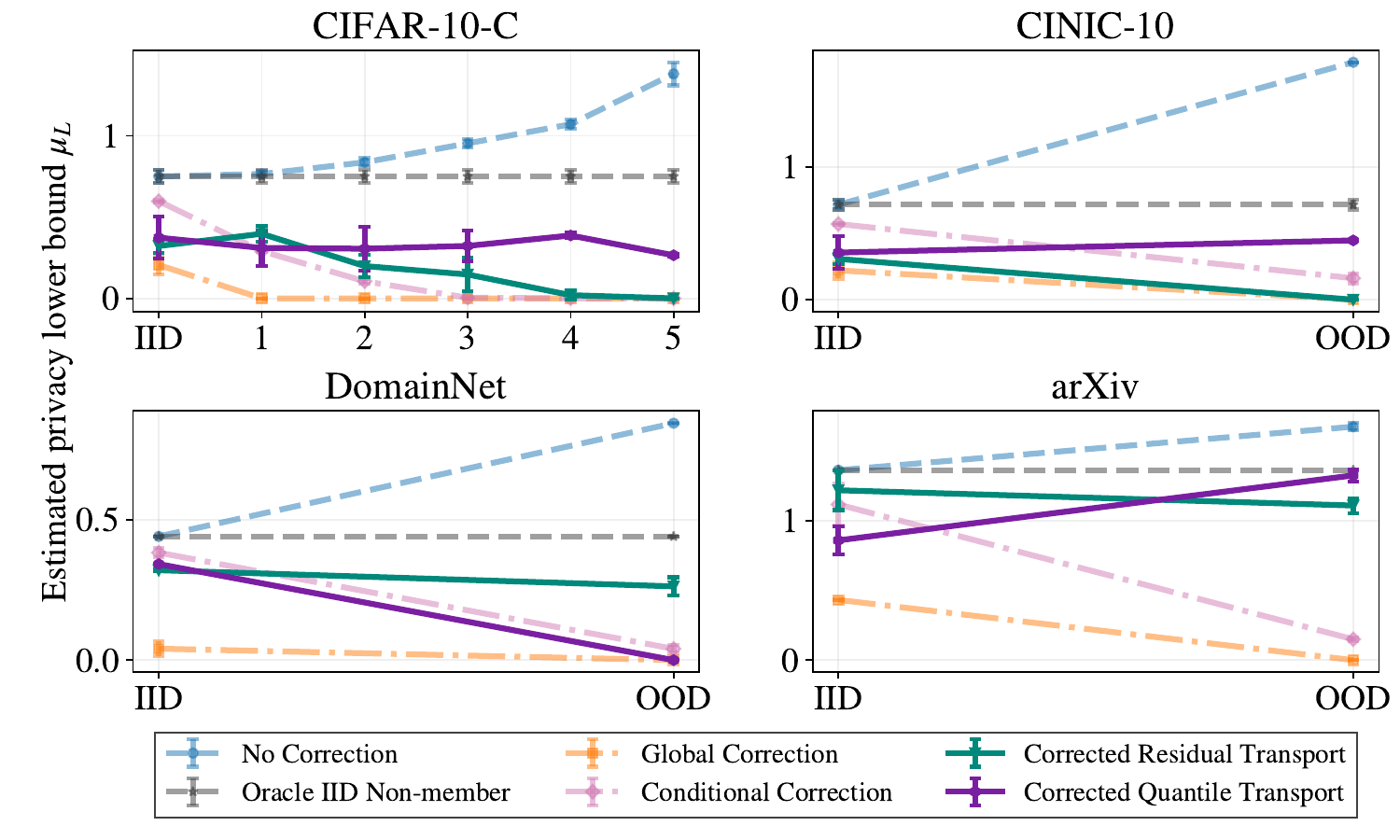}     
    \caption{ {\textbf{Estimated privacy lower bound for all datasets with Bonferroni correction.}
    Our methods provide robust privacy lower bounds and tighter than baselines under OOD shifts. }}
    \label{fig:bonferroni}
    \end{center}
    \vskip -0.2in
\end{figure}

\subsection{ {Running time}} \label{app:running_time}

 {In this section, we show the running time on each benchmark for one audit in \Cref{tab:runtime_linear,tab:runtime_quantile}.
Each audit involves 50 calibration runs and 200 bootstrap repetitions, resulting in a total of\((200+1)\times 50=10,100\) fitted models. 
Importantly, we do not need to retrain the target model. The fitted models are lightweight MLPs used for score transport, rather than the target model itself. 
Each job used a single NVIDIA Tesla V100 GPU. 
Further implementation details are provided in \Cref{app:exp_conf}. 
}

\begin{table}[htbp]
    \centering
    \caption{ {Running time of corrected residual transport on different datasets.}}
    \label{tab:runtime_linear}
    \begin{tabular}{lccc}
        \toprule
        \textbf{Dataset} & \textbf{Time per slab (s) } &
        \textbf{Time per bootstrap (s)} & \textbf{Time per full audit (h) } \\
        \midrule
        CIFAR-10-C (severity=1) & 0.54 & 31.68 & 1.67 \\
        CIFAR-10-C (severity=2) & 0.52 & 32.55 & 1.72 \\
        CIFAR-10-C (severity=3) & 0.51 & 31.58 & 1.67 \\
        CIFAR-10-C (severity=4) & 0.45 & 28.99 & 1.53 \\
        CIFAR-10-C (severity=5) & 0.46 & 30.00 & 1.58 \\
        CINIC-10 & 0.31 & 21.91 & 1.13 \\
        DomainNet & 0.32 & 27.56 & 1.45 \\
        arXiv & 0.90 & 52.93 & 2.66 \\
        \bottomrule
    \end{tabular}
\end{table}

\begin{table}[htbp]
    \centering
    \caption{ {Running time of corrected quantile transport on different datasets.}}
    \label{tab:runtime_quantile}
    \begin{tabular}{lccc}
        \toprule
        \textbf{Datasets} & \textbf{Time per slab (s)} &
        \textbf{Time per bootstrap (s)} & \textbf{Time per full audit (h)} \\
        \midrule
        CIFAR-10-C (severity=1) & 0.58 & 39.27 & 2.01 \\
        CIFAR-10-C (severity=2) & 0.58 & 39.37 & 2.01 \\
        CIFAR-10-C (severity=3) & 0.54 & 36.95 & 1.97 \\
        CIFAR-10-C (severity=4) & 0.56 & 38.32 & 1.99 \\
        CIFAR-10-C (severity=5) & 0.59 & 40.43 & 2.17 \\
        CINIC-10 & 0.42 & 31.40 & 1.66 \\
        DomainNet & 0.54 & 58.78 & 2.75 \\
        arXiv & 0.75 & 45.12 & 2.35 \\
        \bottomrule
    \end{tabular}
\end{table}

\subsection{ {Click-through rate prediction }} \label{app:criteo}
 {
Besides image classification and language tasks presented in \Cref{sec:exps}, here we test on a third domain---click-through rate prediction with Criteo dataset \citep{criteo2015}. 
Specifically, the target learns whether an advertising impression receives a click. 
Membership auditing then asks whether a particular labeled impression was used to train that target.
}

 {
Let one record be
\begin{equation}
   r=(x,y),\qquad x=(x_{\mathrm{num}},x_{\mathrm{cat}}),\quad y\in\{0,1\}, 
\end{equation}
where \(x\) contains 13 numerical and 26 categorical features, and \(y=1\) denotes a click.
}

 {
\textbf{Target training.} 
We use a compact DLRM (Deep Learning Recommendation Model) with 16-dimensional categorical embeddings, an MLP for numerical features.
A top MLP combines the numerical representation and pairwise dot-product interactions into a hidden representation \(h_\theta(x)\in\mathbb{R}^{64}\). 
The predicted click probability is
\begin{equation}
    p_\theta(x)=\Pr_\theta(Y=1\mid x)
          =\sigma\!\left(w^\top h_\theta(x)+b\right).
\end{equation}
where \(\sigma\) denotes the sigmoid function.
}

 {
For the target training set \(M\), containing 500,000 records, we minimize binary cross-entropy:
\begin{equation}
    \widehat L_M(\theta)
=
-\frac{1}{|M|}
\sum_{(x_i,y_i)\in M}
\left[
y_i\log p_\theta(x_i)
+(1-y_i)\log(1-p_\theta(x_i))
\right].
\end{equation}
We train with Adam, learning rate \(10^{-3}\), batch size 2,048, and zero weight decay, for 20 epochs. 
}

 {
\textbf{Source-day splitting.} 
Write \(P_0\) for the joint distribution of labeled impressions on February 15, 2015. 
We sample and deduplicate source-day records, then randomly partition them into disjoint roles:
500,000 target training members \(M\),	
20,000 for target validation,	
10,000 IID audit nonmembers \(N_0^{\mathrm{audit}}\),	
10,000 IID reference nonmembers \(N_0^{\mathrm{ref}}\),	
1,000,000 for auxiliary encoder training	
and 20,000 for auxiliary encoder validation.
}

 {
Within the target training set,
\begin{equation}
     M=M^{\mathrm{audit}}\;\dot\cup\;
M^{\mathrm{ref}}\;\dot\cup\;M^{\mathrm{other}},
\end{equation}
with sizes 10,000, 10,000, and 480,000 respectively. The symbol \(\dot\cup\) denotes a disjoint union. All three subsets participate in target training.
}

 {
``IID nonmember" means a held-out record from the same source-day sampling population:
\begin{equation}
M,\;N_0^{\mathrm{audit}},\;N_0^{\mathrm{ref}}
\quad\text{come from the same source-day population,}
\end{equation}
while their record identities are disjoint.
}

 {
\textbf{Temporal OOD splitting.} Let \(P_\Delta\) denote the joint distribution on a day \(\Delta\) days after the source day. We use
\begin{equation}
\Delta=7:\ \text{February 22},\qquad
\Delta=23:\ \text{March 10}.
\end{equation}
For each later day, we sample 20,000 distinct records and split them into
\begin{equation}
    N_\Delta
=N_\Delta^{\mathrm{audit}}\dot\cup N_\Delta^{\mathrm{ref}},
\qquad
|N_\Delta^{\mathrm{audit}}|
=|N_\Delta^{\mathrm{ref}}|=10,000.
\end{equation}
These records are excluded from target training. 
The paired audit comparisons are therefore
\begin{equation}
   \underbrace{M^{\mathrm{audit}}\ \text{vs.}\ N_0^{\mathrm{audit}}}_{\text{IID audit}},
\qquad
\underbrace{M^{\mathrm{audit}}\ \text{vs.}\ N_\Delta^{\mathrm{audit}}}_{\text{temporal OOD audit}}. 
\end{equation}
Each audit contains 10,000 members and 10,000 nonmembers.
}

 {
Finally, the separately trained auxiliary DLRM supplies the representation \(z=\phi(x)\); the target supplies the membership score
\begin{equation}
    s_\theta(x,y)=-\ell_{\mathrm{BCE}}(p_\theta(x),y).
\end{equation}
Thus the audited object is the complete impression record \((x,y)\), with record-level membership—not an individual’s entire browsing history.
}

\begin{figure}[htbp] 
    \begin{center}          
    \centering
        \includegraphics[width=0.7\textwidth]{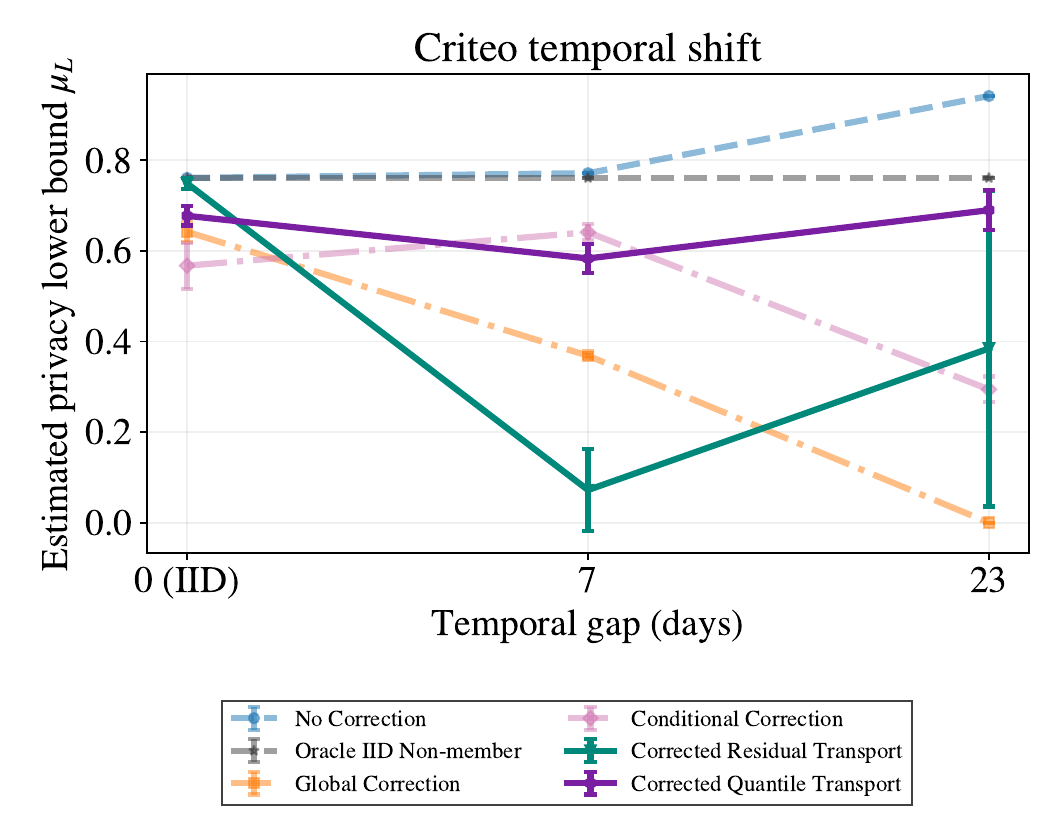}     
    \caption{ {\textbf{Privacy auditing under temporal distribution shift on Criteo.}  Estimated privacy lower bounds \(\mu_L\) for six methods using IID nonmembers (gap 0) and nonmembers sampled 7 and 23 days later. 
    Error bars show ±1 standard deviation across five audit seeds.}}
    \label{fig:criteo_mu}
    \end{center}
    \vskip -0.1in
\end{figure}
 {
The results are presented in \Cref{fig:criteo_mu}. 
Overall, corrected quantile transport provides tighter privacy lower bounds than corrected residual transport in these experiments. 
With a 7-day gap, the distribution shift is relatively mild and has a negligible effect on the audit results. 
With a 23-day gap, the shift becomes more pronounced, and the uncorrected audit substantially overestimates privacy leakage. 
In contrast, our corrected methods produce conservative and tight privacy lower bounds relative to the oracle IID audit.
These additional results further demonstrate the effectiveness of our method in a different application domain and under naturally occurring temporal distribution shifts.
}

\subsection{Ablation results} \label{app:ablation}

\paragraph{Ablation on the conservative correction.}
To assess the role of the conservative correction, we ablate this component and directly audit the transported non-member scores. Results are shown in \Cref{fig:ablation_eta}. 
For comparison, we also report the oracle bound and the bound obtained directly from the OOD non-members. 
Nominal transport alone can debias the shift to a certain level, but still may produce privacy lower bounds above the oracle, showing that score transport by itself does not ensure validity. 
The conservative correction substantially reduces this overestimation by estimating an upper envelope of the transport-induced audit error. 

\begin{figure}[htbp] 
    \begin{center}          
    \centering
        \includegraphics[width=\textwidth]{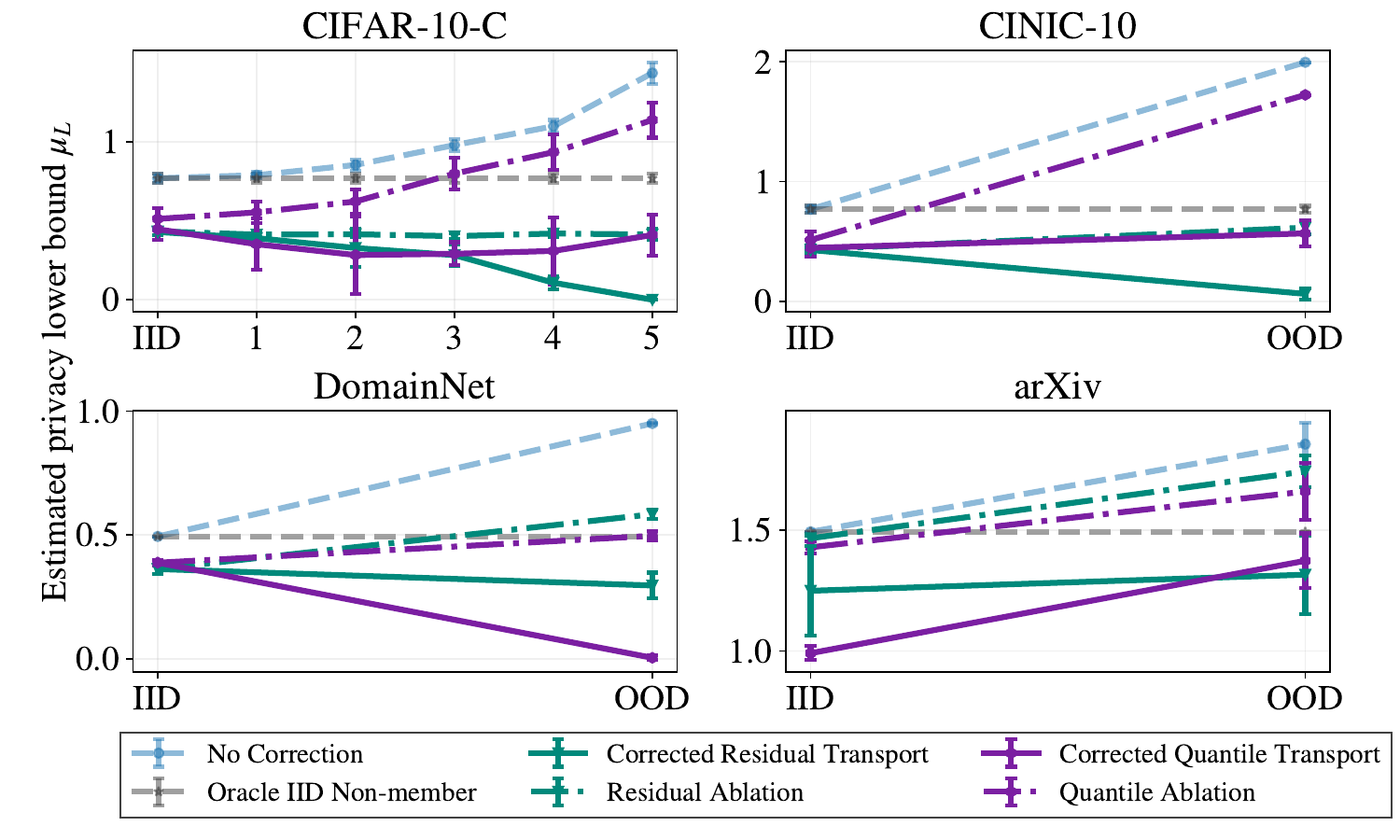}     
    \caption{{Ablation of the conservative correction across datasets. }
    Solid lines are the two proposed methods. Dash lines are the corresponding ablation variants without correction.
    Nominal transport can still exceed the oracle bound, whereas the conservative correction reduces this overestimation.}
    \label{fig:ablation_eta}
    \end{center}
    \vskip -0.2in
\end{figure}
\paragraph{Ablation on the number of fractions.}
Our calibration procedure partitions the available range into multiple fractions to estimate the relationship between distribution shift and auditing error. 
We study the sensitivity of our method to this hyperparameter by varying the number of fractions \(\in\{20,50,100\}\). 
Figure~\ref{fig:ablation_fraction} reports the resulting privacy lower bounds for both residual and quantile transport across all datasets. 
Overall, the estimates and their trends are consistent across the three choices, with only minor variations relative to the uncertainty of the estimates. 
This suggests that the proposed method is not sensitive to the exact number of fractions used in the calibration procedure.
\begin{figure}[htbp] 
    \begin{center}          
    \centering
        \includegraphics[width=\textwidth]{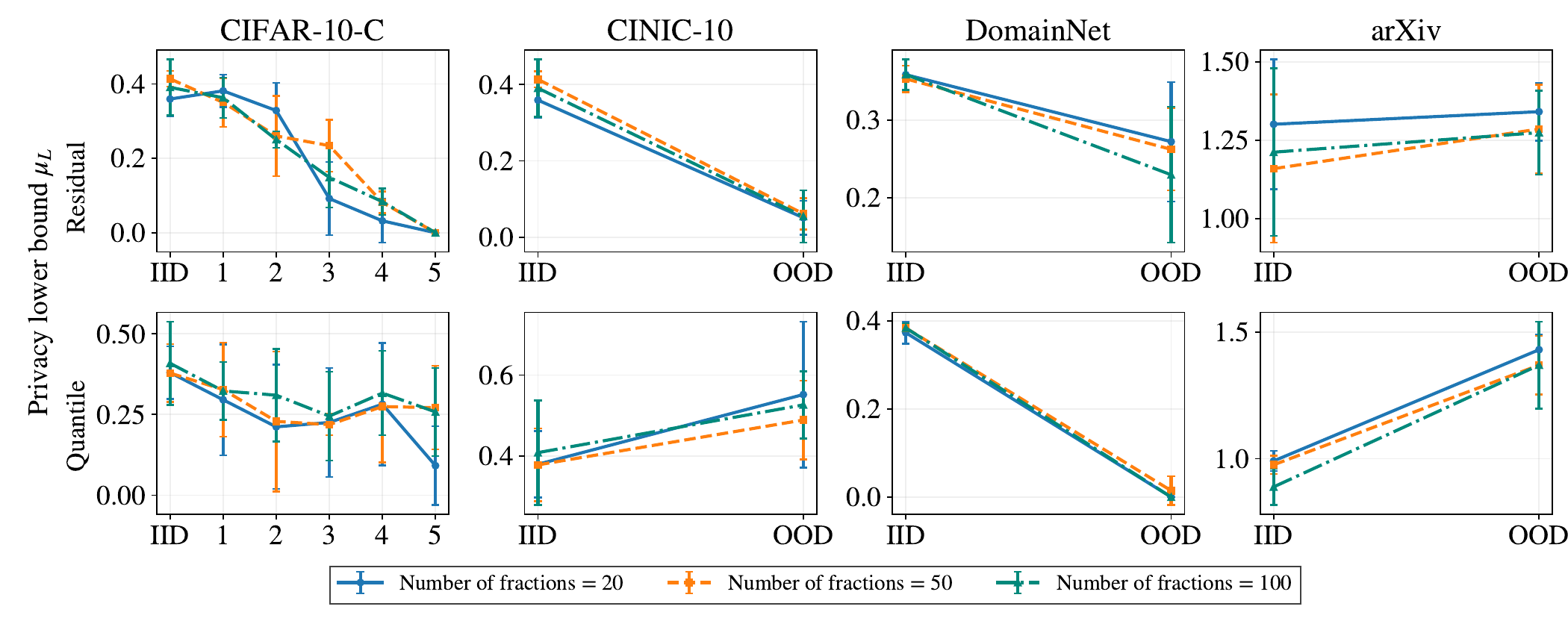}     
    \caption{\textbf{Ablation on number of fractions.} Privacy lower bounds obtained with residual transport (top) and quantile transport (bottom) for number of fractions \(\in\{20,50,100\}\) across all datasets. The estimates remain similar across different choices, indicating that our method is robust to the number of fractions used for calibration. Error bars show standard deviation for 5 seeds.}
    \label{fig:ablation_fraction}
    \end{center}
    \vskip -0.1in
\end{figure}

\section{Additional diagnostics for assessing the assumptions} \label{app:more_diags}

\subsection{ {Diagnostics for residual invariance assumption}} \label{app:test_residual_ass}

 {We compute residuals \(R=S-\widehat m(Z,C)\) using the same predictor, fitted on a separate OOD reference sample, in both domains. 
Let \(D=P\) denote the IID domain and \(D=Q\) the OOD domain. 
Let \(S\) be the auditing score, \(Z\) the feature representation, and \(C\) the class label. 
We fit a score predictor \(\widehat m(Z,C)\) using a separate reference sample of OOD nonmembers. 
We then freeze this predictor and compute residuals on held-out nonmembers from each domain:
\begin{equation}
    R_{d,i}
=
S_{d,i}-\widehat m(Z_{d,i},C_{d,i}),
\qquad d\in\{P,Q\}.
\end{equation}
The same fitted predictor is used for both domains, and neither evaluation sample is used to fit it. 
}

 {
For each class \(c\), define the number of held-out observations in domain \(d\) as
\begin{equation}
    n_{d,c}=\sum_i\mathbf 1\{C_{d,i}=c\}.
\end{equation}
We retain the common class set
\begin{equation}
\mathcal C_\star
=
\{c:n_{P,c}\ge20,\;n_{Q,c}\ge20\}.
\end{equation}
For every retained class, we estimate its residual CDF separately in each domain:
\begin{equation}
\widehat F_{d,c}(t)
=
\frac{1}{n_{d,c}}
\sum_{i:C_{d,i}=c}
\mathbf 1\{R_{d,i}\le t\}.
\end{equation}
Thus, \(\widehat F_{d,c}(t)\) is the fraction of class-\(c\) residuals at or below threshold \(t\).
We estimate the IID class proportions, conditional on the retained classes, by
\begin{equation}
\widehat\pi_c
=
\frac{n_{P,c}}
{\sum_{k\in\mathcal C_\star}n_{P,k}},
\qquad c\in\mathcal C_\star.
\end{equation}
}

\begin{figure}[t] 
    \begin{center}          
    \centering
        \includegraphics[width=\textwidth]{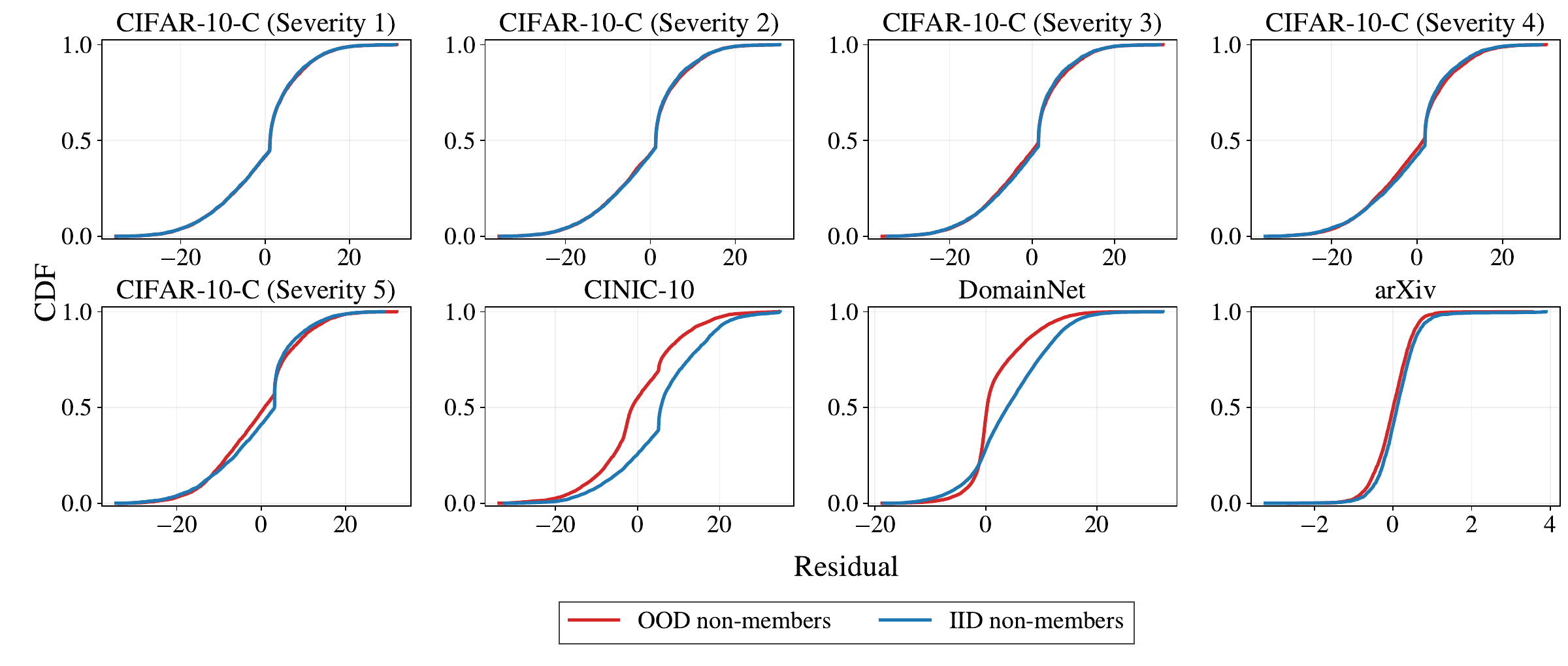}     
    \caption{ {\textbf{Residual distributions across domains.} Empirical CDFs of held-out nonmember residuals \(R=S-\widehat m(Z,C)\), computed using the same score predictor fitted on a separate OOD reference sample. Separation indicates aggregate distribution differences.}}
    \label{fig:residual_cdf}
    \end{center}
    \vskip -0.2in
\end{figure}

 {
The two plotted curves are then
\begin{equation}
\widehat F_d(t)
=
\sum_{c\in\mathcal C_\star}
\widehat\pi_c\,\widehat F_{d,c}(t),
\qquad d\in\{P,Q\}.
\end{equation}
}

 {
In \Cref{fig:residual_cdf}, the blue curve, \(\widehat F_P\), equals the ordinary empirical CDF of the retained IID observations. 
The red curve, \(\widehat F_Q\), reweights OOD observations so that each class has the same total weight as in the IID sample. 
Equivalently, every observation from class \(c\) in domain \(d\) contributes weight
\begin{equation}
w_{d,i}=\frac{\widehat\pi_c}{n_{d,c}}.
\end{equation}
This construction prevents differences in class proportions alone from producing a gap between the curves.
We report their maximum vertical separation:
\begin{equation}
\widehat\Delta
=
\sup_{t\in\mathbb R}
\left|
\widehat F_P(t)
-
\widehat F_Q(t)
\right|.
\end{equation}
This is a descriptive KS distance, rather than a \(p\)-value. 
For example, DomainNet’s value of \(0.263\) means that, at some residual threshold, the two weighted cumulative probabilities differ by approximately 26.3 percentage points. 
The corresponding distances are \(0.310\) for CINIC-10, \(0.098\) for arXiv, and \(0.010\)–\(0.072\) across the five brightness severities. 
}

 {
To connect this visualization to the assumption, a class-conditional formulation of the strong residual assumption is
\begin{equation}
\mathcal L(R\mid Z=z,C=c,D=d)=G_c
\end{equation}
for all relevant \(z\) and \(d\). 
This implies the weaker, domain-marginal condition
\begin{equation}
F_{P,c}(t)=F_{Q,c}(t)
\quad\text{for all }c,t,
\end{equation}
which, in turn, implies equality of the two population CDFs under any common class weights. 
The figure visualizes this necessary marginal consequence. 
}

\subsection{ {Sensitivity to the calibration displacement range}} \label{app:target_exceeds}

 {
As stated in \Cref{app:cor}, when the target displacement exceeds the maximum value covered by the calibration tasks, we acknowledge that our current assumptions do not provide theoretical guarantees in this setting.
Because it requires extrapolation beyond the displacement range covered by the calibration tasks. 
To investigate this case empirically, we specifically design calibration tasks such that the target displacement exceeds the maximum displacement observed during calibration. 
} 

 {
For each audit and bootstrap replicate \(b\), let
\begin{equation}
\mathcal C_b=\{(\Delta_{bj},e_{bj})\}_{j=1}^{K},\qquad K=50,
\end{equation}
denote the saved calibration points, where \(\Delta_{bj}\) is the displacement and \(e_{bj}\) is the corresponding nonnegative calibration error. 
Let \(\pi_b\) order the points by increasing displacement:
\begin{equation}
\Delta_{b,\pi_b(1)}\le \cdots\le \Delta_{b,\pi_b(K)}.
\end{equation}
For a retention fraction \(\rho\in\{1/2,2/3\}\), define
\begin{equation}
m_\rho=\lfloor \rho K\rfloor,
\qquad
\mathcal C_b^{(\rho)}
=
\{(\Delta_{b,\pi_b(j)},e_{b,\pi_b(j)})\}_{j=1}^{m_\rho}.
\end{equation}
Thus, the two calibrations retain the 25 and 33 smallest-displacement points, respectively, resulting in narrower displacement ranges.
}

 {
Using the retained points, we fit a linear \(95\)th-percentile quantile regression:
\begin{equation}
(\widehat a_b^{(\rho)},\widehat s_b^{(\rho)})
\in
\underset{a\in\mathbb R,\;s\ge0}{\arg\min}
\sum_{j=1}^{m_\rho}
\ell_{0.95}\!\left(
e_{b,\pi_b(j)}-a-s\Delta_{b,\pi_b(j)}
\right),
\end{equation}
where
\begin{equation}
\ell_\tau(u)=u\bigl(\tau-\mathbf 1\{u<0\}\bigr)
\end{equation}
is the pinball loss. The nonnegative slope constraint ensures that the fitted error does not decrease with displacement.
}

 {
At the audit’s target displacement \(d^*\), the fitted error is
\begin{equation}
\widehat r_b^{(\rho)}
=
\left[\widehat a_b^{(\rho)}
+\widehat s_b^{(\rho)}d^*\right]_+,
\qquad [x]_+=\max(x,0).
\end{equation}
For the final audit bound, we pair each bootstrap’s transported bound \(\widehat\mu_{b,\mathrm{trans}}\) with its own calibrated error:
\begin{equation}
\mu_{b,\text{corr}}^{(\rho)}
=
\left[\widehat\mu_{b,\mathrm{trans}}
-\widehat r_b^{(\rho)}\right]_+.
\end{equation}
Finally we apply median bias correction and report with \Cref{alg:bootstrap}.
}

 {
\Cref{fig:linear_target_cali,fig:quantile_target_cali} report the comparison of estimated privacy bounds across narrower displacement calibration. 
Using smaller displacements does not significantly increase or decrease the final audit bound, which shows the
robustness of our method in this extrapolation setting. 
Across all experiments, the predicted audit errors conservatively dominate the true audit errors. 
\Cref{fig:cal_fit_1/2,fig:cal_fit_2/3} provide empirical evidence supporting Assumption \ref{ass:axis_sufficiency} in this setting. 
}

\begin{figure}[htbp] 
    \begin{center}          
    \centering
        \includegraphics[width=\textwidth]{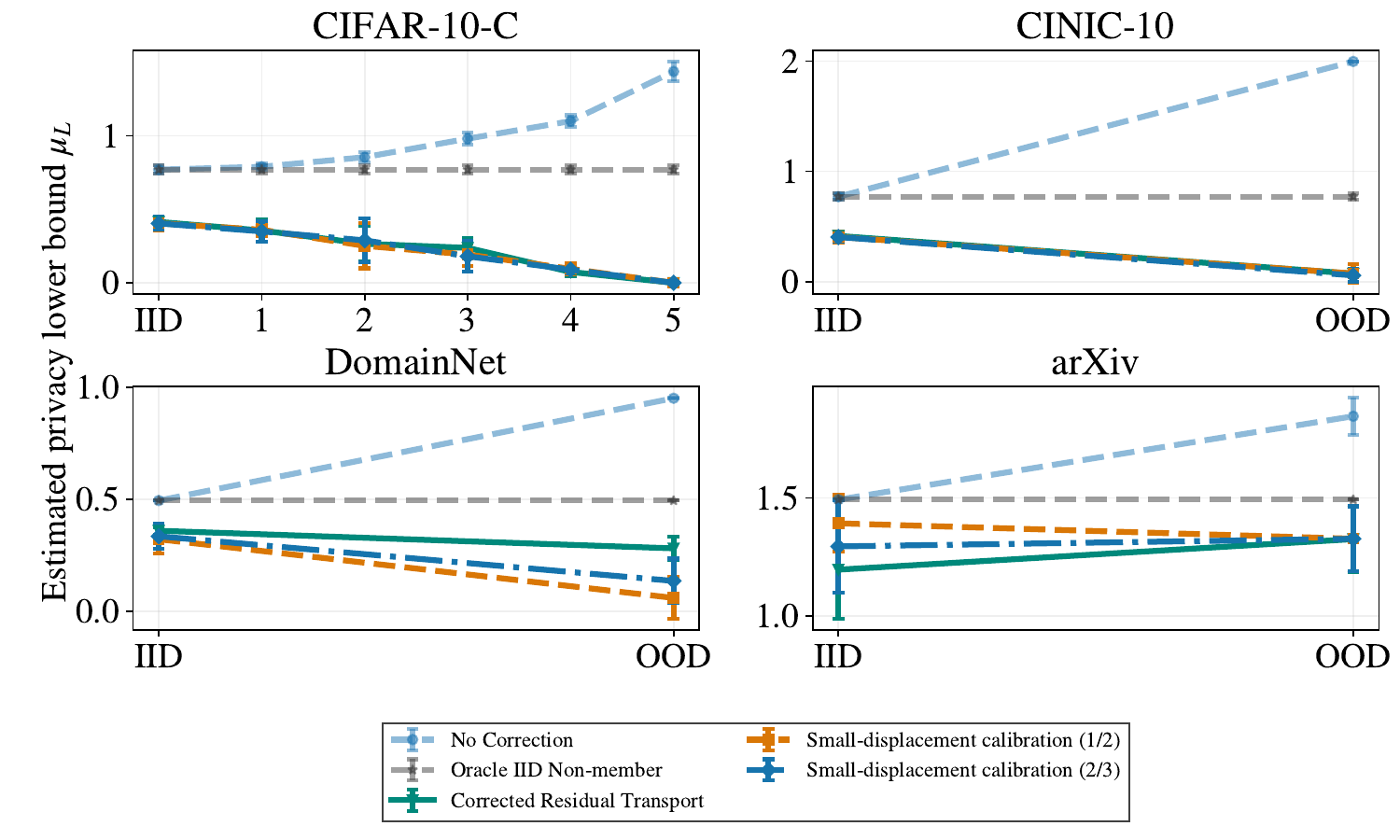}     
    \caption{ {\textbf{Effect of calibration range on corrected residual transport audits.} It compares final audit \(\mu\) using all 50 calibration points, the smallest 25, and the smallest 33. The strongest change occurs on DomainNet: mean OOD \(\mu\) decreases from 0.281 to 0.060 and 0.136, respectively. On the other domains, the three estimates are nearly identical. }}
    \label{fig:linear_target_cali}
    \end{center}
    \vskip -0.2in
\end{figure}

\begin{figure}[htbp] 
    \begin{center}          
    \centering
        \includegraphics[width=\textwidth]{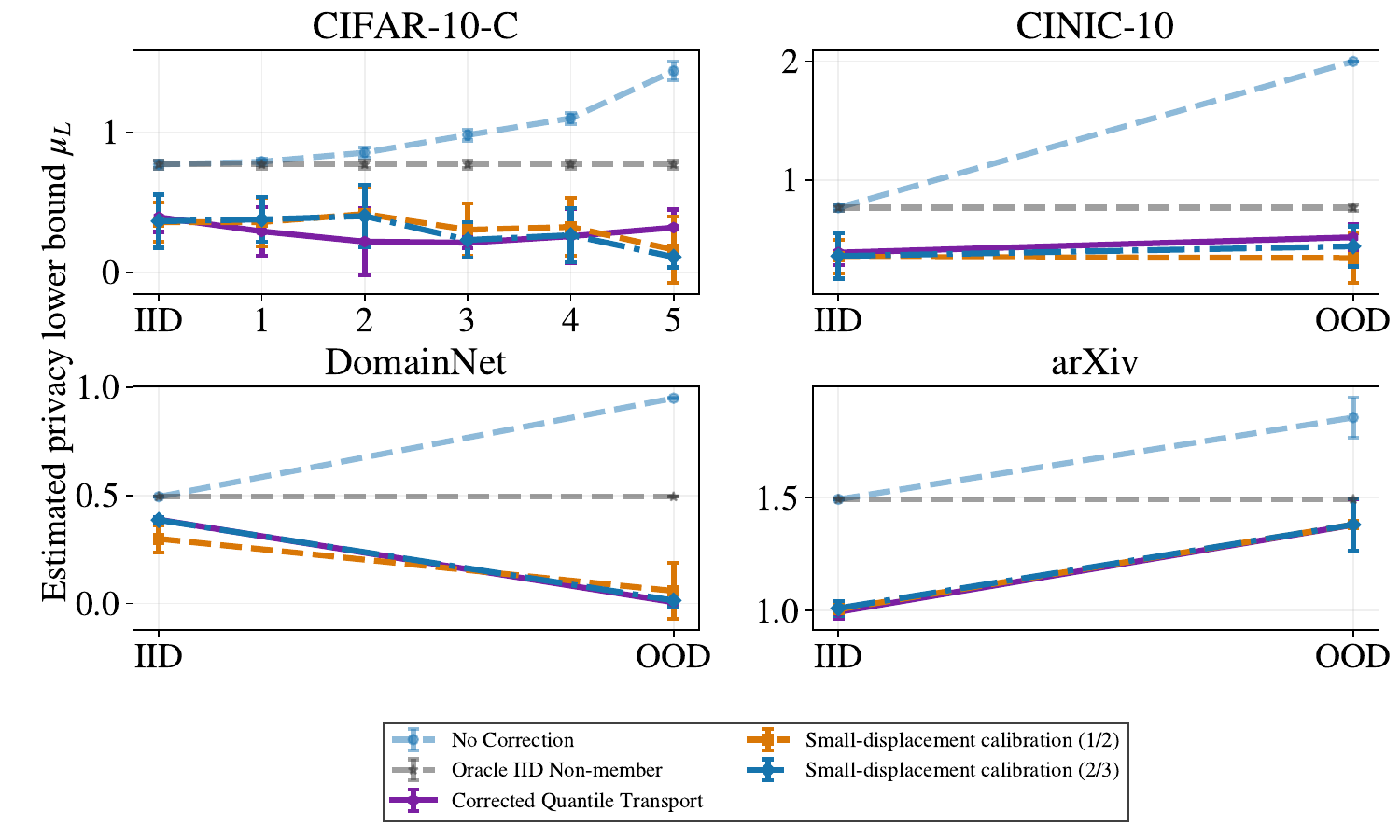}     
    \caption{ {\textbf{Effect of calibration range on corrected quantile transport audits.} Using smaller displacements does not consistently increase or decrease the final audit bound, which shows the robustness of our method.}}
    \label{fig:quantile_target_cali}
    \end{center}
    \vskip -0.2in
\end{figure}

\begin{figure}[htbp] 
    \begin{center}          
    \centering
        \includegraphics[width=\textwidth]{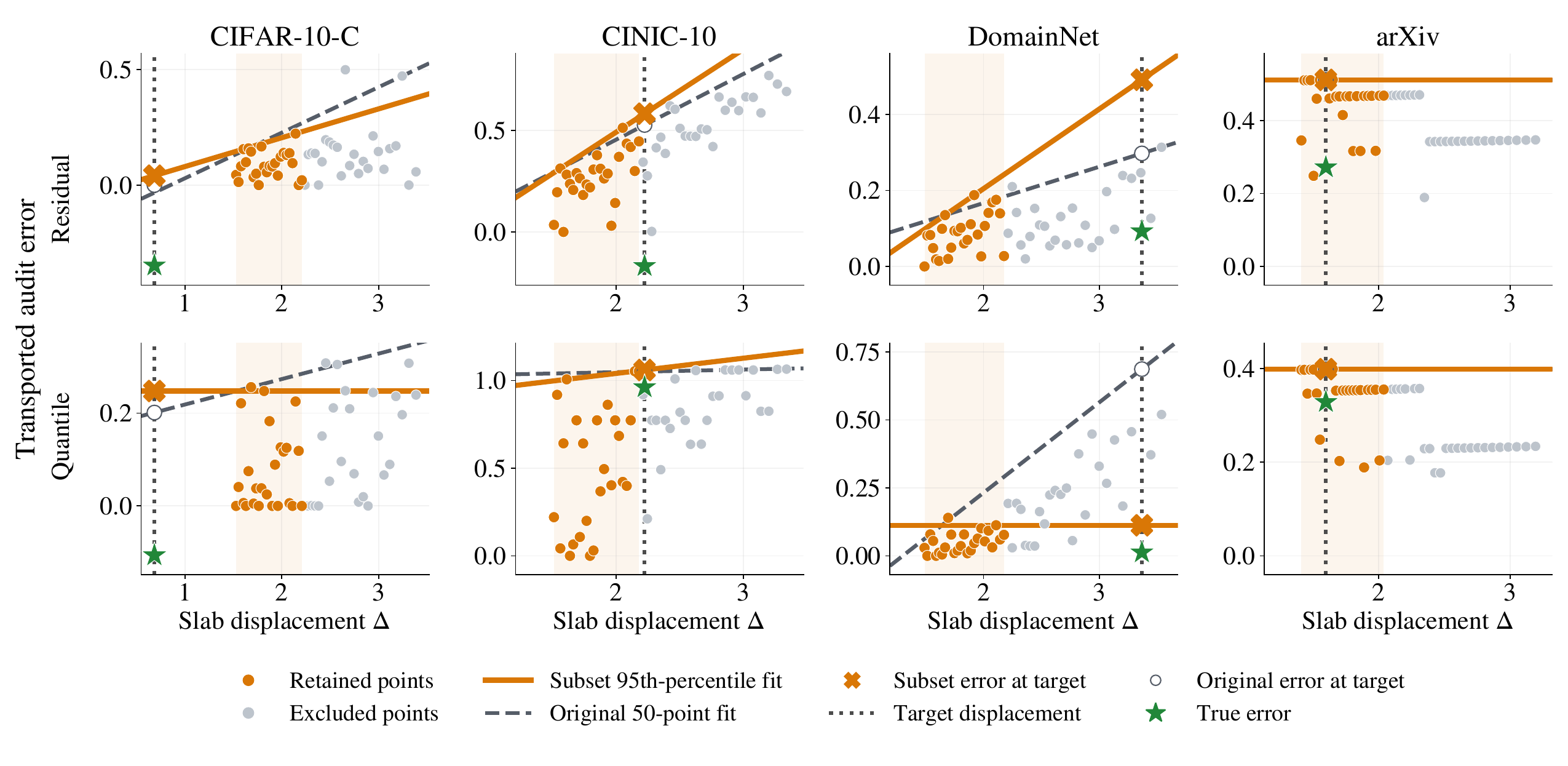}     
    \caption{ {\textbf{Calibration fits using the smallest half of displacements.} It illustrates how retaining the half of the data with the smallest values affects the fitted regression and its prediction at the target displacement. The changes are substantial for DomainNet, while the results remain largely unchanged for the other datasets.}}
    \label{fig:cal_fit_1/2}
    \end{center}
    \vskip -0.2in
\end{figure}

\begin{figure}[htbp] 
    \begin{center}          
    \centering
        \includegraphics[width=\textwidth]{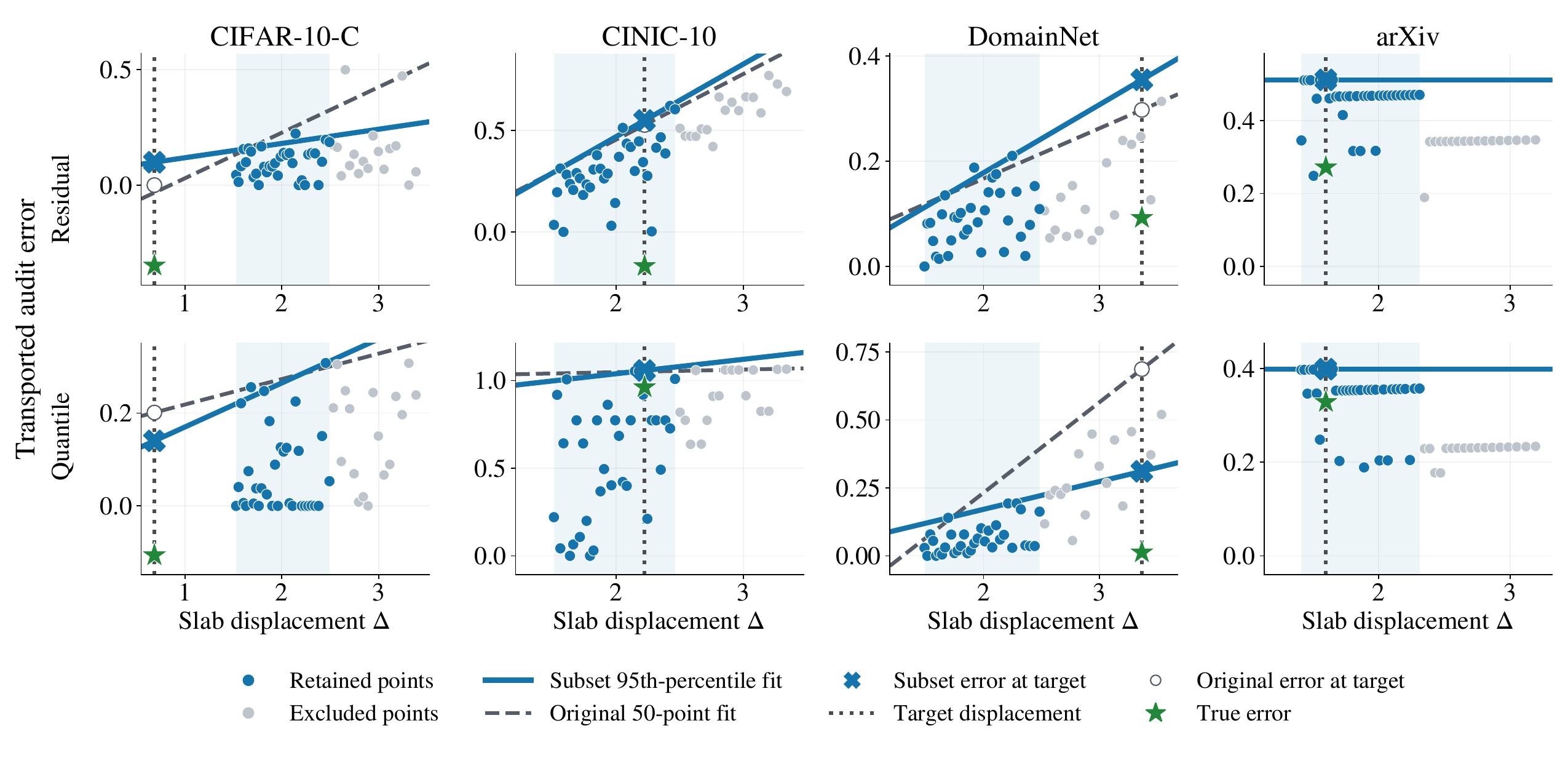}     
    \caption{ {\textbf{Calibration Fits Using the Smallest Two-Thirds of Displacements.}  It illustrates how retaining the two-thirds of the data with the smallest values affects the fitted regression and its prediction at the target displacement. The changes are substantial for DomainNet, while the results remain largely unchanged for the other datasets.}}
    \label{fig:cal_fit_2/3}
    \end{center}
    \vskip -0.2in
\end{figure}

\subsection{ {Robustness and failure modes under combined mean and covariance distribution shift}} \label{app:control_shift}

 {
In this section, we design synthetic experiments \citep{cebere2026privacy} to control the shift level and demonstrate under what cases our method works and fails.
}


 {
First, we sample a random unit direction in \(d=10^6\) dimensions:
\begin{equation}
v\sim\mathcal N(0,I_d),\qquad u=\frac{v}{\|v\|_2}.
\end{equation}
Training records and independent IID nonmembers come from the same population:
\begin{equation}
\widetilde X\sim\mathcal N(u,I_d),
\qquad
X=\frac{\widetilde X}{\|\widetilde X\|_2}.
\end{equation}
Thus every observed record has unit norm. We generate \(100{,}000\) training records and use \(50{,}000\) of them as audit members, alongside \(50{,}000\) independent audit nonmembers.
}

 {
To introduce shift, we change the raw nonmember distribution to
\begin{equation}
\widetilde Y_t\sim
\mathcal N\!\left((1-t)u,\;
I_d+\frac{t^2}{4}uu^\top\right),
\qquad
Y_t=\frac{\widetilde Y_t}{\|\widetilde Y_t\|_2}.
\end{equation}
An equivalent construction makes the two controls explicit. Write a standard Gaussian draw as
\begin{equation}
\varepsilon=g\,u+\varepsilon_\perp,
\qquad
g=u^\top\varepsilon,\quad
\varepsilon_\perp=(I_d-uu^\top)\varepsilon.
\end{equation}
Then
\begin{equation}
\widetilde Y_t
=
\underbrace{(1-t)u}_{\text{shifted mean}}
+
\underbrace{\sqrt{1+t^2/4}\,g\,u}_{\text{increased variation along }u}
+
\underbrace{\varepsilon_\perp}_{\text{unchanged orthogonal noise}}.
\end{equation}
Consequently, \(t\) controls both:
\begin{itemize}
    \item \textbf{Mean shift}: the raw mean moves by \(-tu\), a distance \(t\) from the IID mean for \(t\geq0\).
    \item \textbf{Covariance shift}: variance along \(u\) increases from \(1\) to \(1+t^2/4\); variance in every direction orthogonal to \(u\) remains \(1\).
\end{itemize}
}

 {
At \(t=0\), nonmembers are IID with the training population. 
Larger \(t\) jointly increases the mean displacement and directional variance. 
}

 {
We abstract the training of the machine learning model by computing a differentially private sum of the members' feature vectors. 
Under the standard \textit{add/remove} adjacency definition, the global \(L_2\) sensitivity of the sum query is strictly bounded by the maximum norm of a single vector:
\begin{equation}  
\Delta_2=\max_X\|X\|_2=1.
\end{equation}
We apply the Gaussian mechanism by adding isotropic Gaussian noise calibrated to a scale \(\sigma>0\):
\begin{equation}
\theta_{\mathrm{priv}}
=\sum X+\sigma\xi,
\qquad
\xi\sim\mathcal N(0,I_d),\quad \sigma=1.5,
\end{equation}
and each record receives the audit score
\begin{equation}
s(x)=x^\top\theta_{\mathrm{priv}}.
\end{equation}
Releasing \(\theta_{\mathrm{priv}}\) rigorously guarantees \(\mu\)-GDP with a true algorithmic privacy limit bounded at \(\mu_{\text{oracle}}=\Delta_2/\sigma=1/\sigma=0.67.\) 
}

 {
As shown in the \Cref{fig:synthetic_mu}, corrected residual transport produces tighter privacy lower bounds than corrected quantile transport for shift levels up to \(t=16\). 
However, at \(t=64\), neither method produces a meaningful estimate. 
To better understand this failure, we additionally visualize the domain-propensity distributions in \Cref{fig:synthetic_propensity}. 
The results show that the overlap between the two distributions becomes extremely limited at \(t=64\), suggesting that severe support mismatch contributes to the failure of our calibration procedure.
}

\begin{figure}[htbp] 
    \begin{center}          
    \centering
        \includegraphics[width=0.7\textwidth]{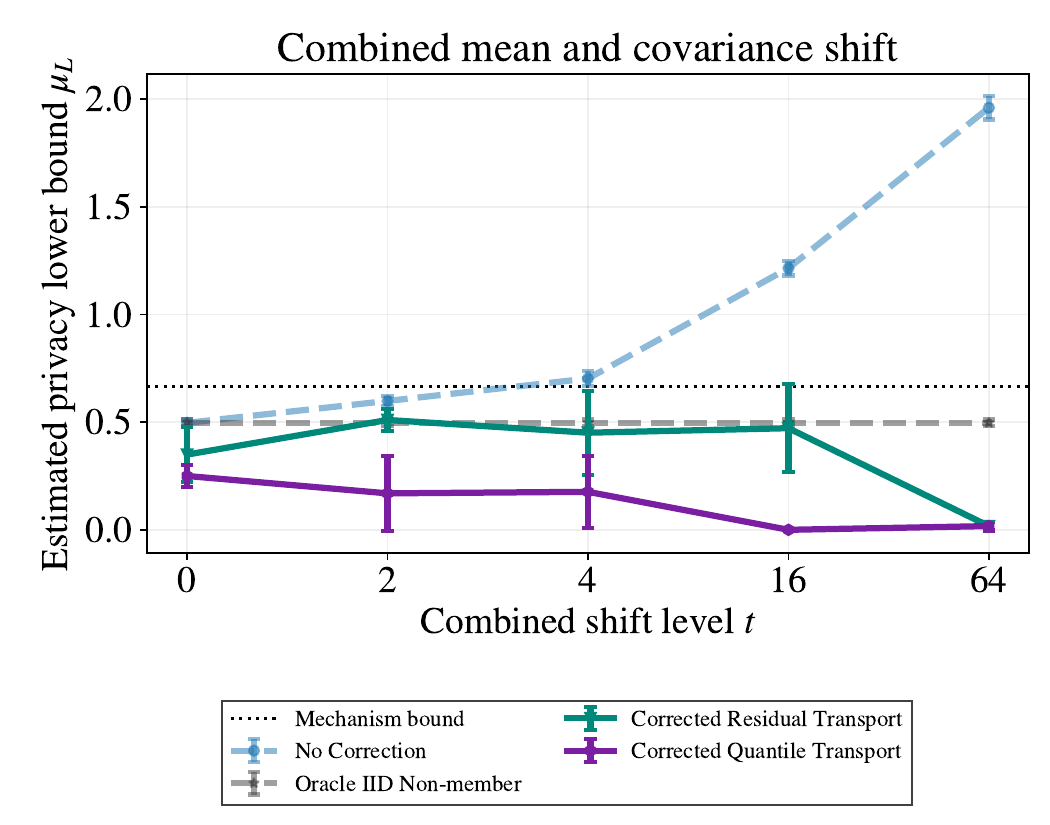}     
    \caption{ {\textbf{Privacy auditing under combined mean and covariance shift. } Estimated privacy lower bounds \(\mu_L\) for uncorrected auditing, corrected residual transport, and corrected quantile transport, compared with the IID oracle and mechanism bound. 
    Error bars show one sample standard deviation across five seeds.
    Increasing \(t\) changes both the mean and variance. 
    Corrected residual transport produces tighter privacy lower bounds than corrected quantile transport for shift levels up to \(t=16\). 
    However, at \(t=64\), neither method produces a meaningful estimate. 
    }}
    \label{fig:synthetic_mu}
    \end{center}
    \vskip -0.2in
\end{figure}

\begin{figure}[htbp] 
    \begin{center}          
    \centering
        \includegraphics[width=\textwidth]{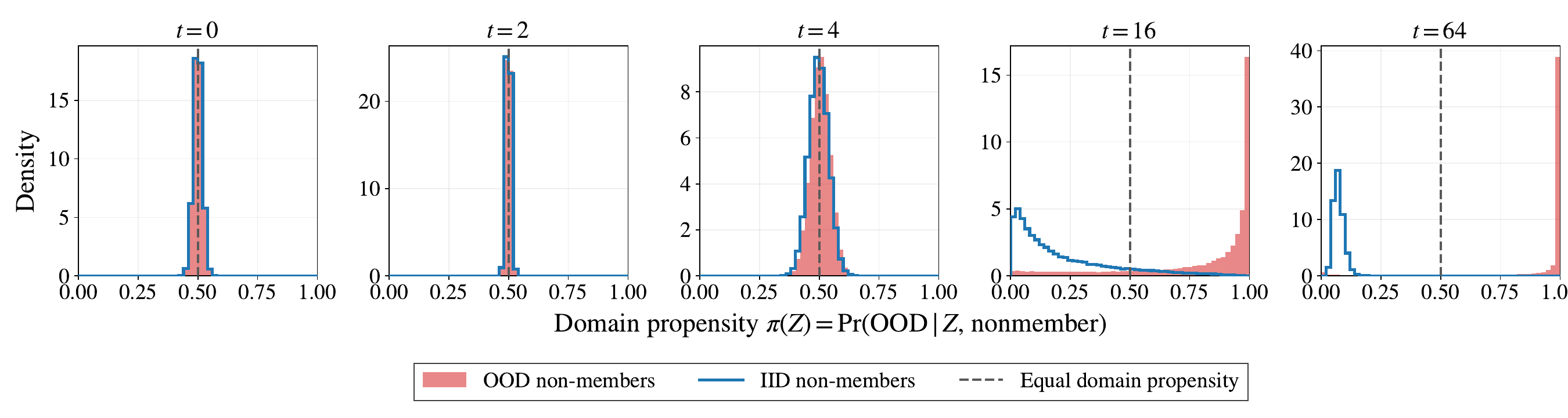}     
    \caption{ {\textbf{Domain-propensity distributions across synthetic shift levels.}  
    Distributions of \(\pi(Z)=\Pr(\mathrm{OOD}\mid Z,\mathrm{nonmember})\) for IID nonmembers (blue outlines) and shifted nonmembers (red shading), evaluated separately at each shift level using seed 0. 
    Histograms use 50 equal-width bins; the dashed line marks \(\pi(Z)=0.5\). 
    Overlapping distributions indicate limited separation by the fitted classifier, whereas concentration toward opposite endpoints indicates stronger domain separation.}}
    \label{fig:synthetic_propensity}
    \end{center}
    \vskip -0.2in
\end{figure}

\subsection{ {Bootstrap-based uncertainty estimation for Assumption \ref{ass:uniform_coverage}}} \label{app:diss_ass4} 
 {
One possible approach to establishing Assumption \ref{ass:uniform_coverage} is to bootstrap the reference data, reconstruct the calibration tasks, refit the conditional quantile curve, and construct a simultaneous upper confidence band using the maximum deviation across the interval. 
Our current implementation bootstraps the entire auditing pipeline, including reference-data resampling, transport fitting, calibration, and correction. 
It therefore captures resampling variability in the final corrected estimate, including variability introduced by calibration.
}

\subsection{ {Geometry and interpretation of source-to-target displacement}} \label{app:explain_ood-to-iid}

 {
Here we provide the detailed explanation of the source-to-target displacement in \Cref{sec:correction}.
Recall the definition as following:}

 {
\textbf{Source-to-target displacement.}
We reduce the source-to-target mismatch to a scalar. 
Let \(\widehat\Sigma_Q\) be a Ledoit--Wolf shrinkage estimate of the reference covariance \(Z_N^{Q,\text{ref}}\) and \(W = \widehat\Sigma^{-1/2}_Q\) the corresponding whitening map, so that one unit is one reference standard deviation. 
 {We whiten the representations using the OOD covariance to account for differences in feature scale and correlation, so that the shift direction is not dominated simply by high-variance dimensions.}
With \(\bar z_{N}^Q\) and \(\bar z_{M}^P\) the reference OOD non-member and member means, define the \emph{shift axis} and the induced displacements
\begin{equation}
  \mathcal O \;=\; \frac{W(\bar z_{M}^P - \bar z_{N}^Q)}
                {\|W(\bar z_{M}^P - \bar z_{N}^Q)\|},
  \qquad
  d( z) \;=\; \bigl\langle W (z - \bar z_{N}^Q),\; \mathcal O \bigr\rangle,
  \qquad
  d^* \;=\; \bigl\|W(\bar z_{M}^P - \bar z_{N}^Q)\bigr\| . \tag{17}
\end{equation}
\(\mathcal O\) means the shift direction from non-members to members.
Each non-member sample is assigned a signed displacement \(d\), centered to have mean zero, while \(d^*\) denotes the mean displacement of the member pool, representing the distance over which the transport is required to extrapolate.
}

 {
\textbf{1. What whitening does.}
}

 {
The magnitude of a displacement should be interpreted relative to the variability of the source population. For example, if the reference covariance were diagonal, \(\widehat\Sigma_Q=\operatorname{diag}(\sigma_1^2,\ldots,\sigma_p^2)\), then a mean difference \(\bar Z=\bar z_M^P-\bar z_N^Q\) would have standardized magnitude
\begin{equation}
\|W\bar Z\|^2
=\sum_{k=1}^{p}\frac{\bar Z_k^2}{\sigma_k^2}.
\end{equation}
A difference in a highly variable feature therefore contributes less than the same difference in a stable feature. For a general covariance matrix, whitening additionally accounts for correlations:
\begin{equation}
\|W\bar Z\|^2
=\bar Z^\top\widehat\Sigma_Q^{-1}\bar Z.
\end{equation}
Thus, the geometry measures departure relative to the joint variation of the reference OOD non-members.
}

 {
We use the source covariance because the source population supplies the observations from which transport is learned and calibrated. Its covariance defines a fixed reference against which target displacement is assessed. 
Ledoit–Wolf shrinkage \citep{ledoit2004well},
\begin{equation}
\widehat\Sigma_Q=(1-\rho)S_Q+\rho\tau I,
\qquad
\tau=\operatorname{tr}(S_Q)/p,
\end{equation}
stabilizes this geometry: regularizing small covariance eigenvalues prevents their inverse square roots from excessively amplifying poorly estimated directions. 
\(S_Q\) denotes the empirical covariance matrix of the reference OOD non-member representations.
Accordingly, distances are expressed in units determined by the regularized source covariance, approximately corresponding to source standard-deviation units.
}

 {
\textbf{2. What the shift axis means.}
}

 {
After whitening, the mean difference is \(W\bar Z\). 
We decompose it into a direction and a magnitude:
\begin{equation}
W\bar Z=d^*\mathcal O,
\qquad
\mathcal O=\frac{W\bar Z}{\|W\bar Z\|},
\qquad
d^*=\|W\bar Z\|.
\end{equation}
The axis \(\mathcal O\) selects the direction most aligned with the observed source-to-target mean change. More precisely, for any unit vector \(u\), the difference between the projected reference means is
\begin{equation}
\left\langle W\bar z_M^P,u\right\rangle
-
\left\langle W\bar z_N^Q,u\right\rangle
=\langle W\bar Z,u\rangle
\leq \|W\bar Z\|,
\end{equation}
with equality at \(u=\mathcal O\). Hence, \(\mathcal O\) maximizes the signed mean separation among unit directions in whitened space.
}

 {
This choice gives a principled one-dimensional summary: it retains the entire whitened mean difference,
\begin{equation}
(I-\mathcal O\mathcal O^\top)W\bar Z=0,
\end{equation}
while discarding variation orthogonal to that direction. The reduction is therefore exact for the difference between the means, although it does not preserve all differences between the distributions.
}

 {
\textbf{3. What an individual displacement means.}
}

 {
For a representation \(z\), let
\begin{equation}
x(z)=W(z-\bar z_N^Q).
\end{equation}
Its decomposition relative to the shift axis is
\begin{equation}
x(z)=d(z)\mathcal O+r(z),
\qquad
d(z)=\langle x(z),\mathcal O\rangle,
\qquad
\langle r(z),\mathcal O\rangle=0.
\end{equation}
Thus, \(d(z)\) is the component of the sample’s reference-centered representation along the mean-shift direction; \(r(z)\) contains its orthogonal component. In particular,
\begin{equation}
\|x(z)\|^2=d(z)^2+\|r(z)\|^2.
\end{equation}
This distinguishes displacement along the selected axis from total distance to the source mean.
}

 {
The sign of \(d(z)\) records orientation: positive displacement points toward the member mean, while negative displacement points in the opposite direction. Centering and normalization give
\begin{equation}
\frac{1}{n_N}\sum_{i=1}^{n_N}d(z_{N,i}^Q)=0,
\qquad
\frac{1}{n_M}\sum_{j=1}^{n_M}d(z_{M,j}^P)=d^*.
\end{equation}
Consequently, source records and target records are placed on a common coordinate system whose population means are \(0\) and \(d^*\). Individual coordinates describe location along this axis; they do not determine membership.
}

\subsection{Additional diagnostics for conditional transportability} \label{app:more_con_trans}
We provide the full set of PIT-based diagnostics for the conditional transportability across datasets and distribution shifts. 
Following the analysis in \Cref{sec:diagnostic}, we partition samples according to the estimated domain propensity score \(\hat{\pi}(Z)\) and compare the empirical CDFs of the PIT values for IID and OOD non-members within each propensity bin. 
Under conditional invariance, the two distributions should be similar within regions of common support; deviations between the red and blue curves indicate potential violations of this assumption. 
The dashed diagonal shows the uniform reference distribution, while gray panels indicate propensity bins with insufficient common support for a reliable comparison. Figures~\ref{fig:pit_cifar10c} and~\ref{fig:pit_additional} show the results across CIFAR-10-C corruption severities and the remaining benchmark settings, respectively.
The results indicate violations of the conditional invariance assumption for CINIC-10 and DomainNet. 
IID non-member curves below the diagonal indicate that the observed MIA scores are larger than those predicted from the OOD reference distribution. 
For CIFAR-10-C and arXiv, the gray panels indicate insufficient common support, suggesting that transport in these regions may require extrapolation.

 \begin{figure}[htbp] 
        \begin{center}
            \begin{minipage}[t]{0.9\textwidth}
                \centering
                \includegraphics[width=\textwidth]{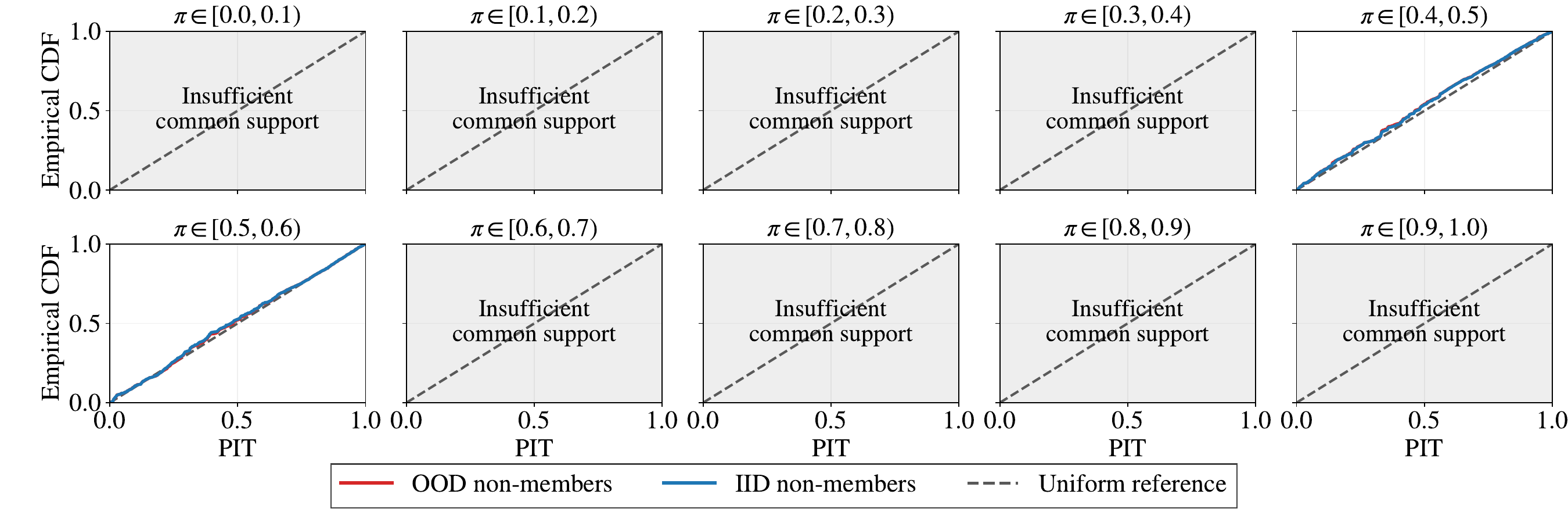}
                \subcaption{CIFAR-10-C with severity level: 1}
            \end{minipage}
            \begin{minipage}[t]{0.9\textwidth}
                \centering
                \includegraphics[width=\textwidth]{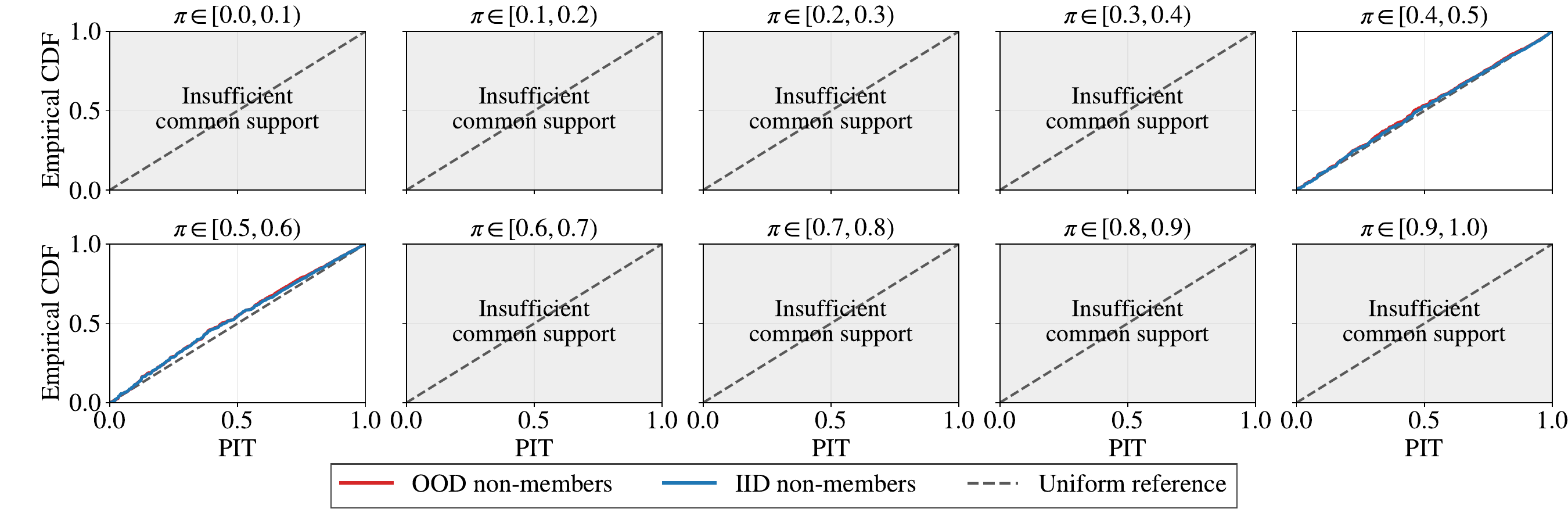}
                 \subcaption{CIFAR-10-C with severity level: 2}
            \end{minipage}
             \begin{minipage}[t]{0.9\textwidth}
                \centering
                \includegraphics[width=\textwidth]{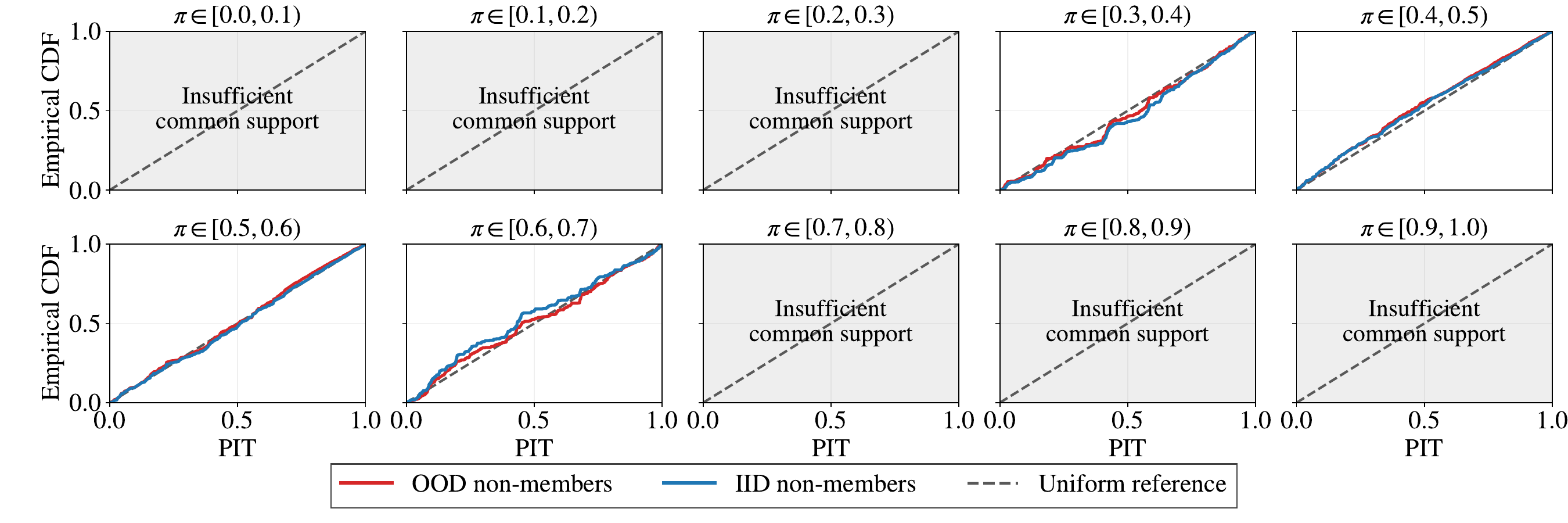}
                 \subcaption{CIFAR-10-C with severity level: 3} 
            \end{minipage}
             \begin{minipage}[t]{0.9\textwidth}
                \centering
                \includegraphics[width=\textwidth]{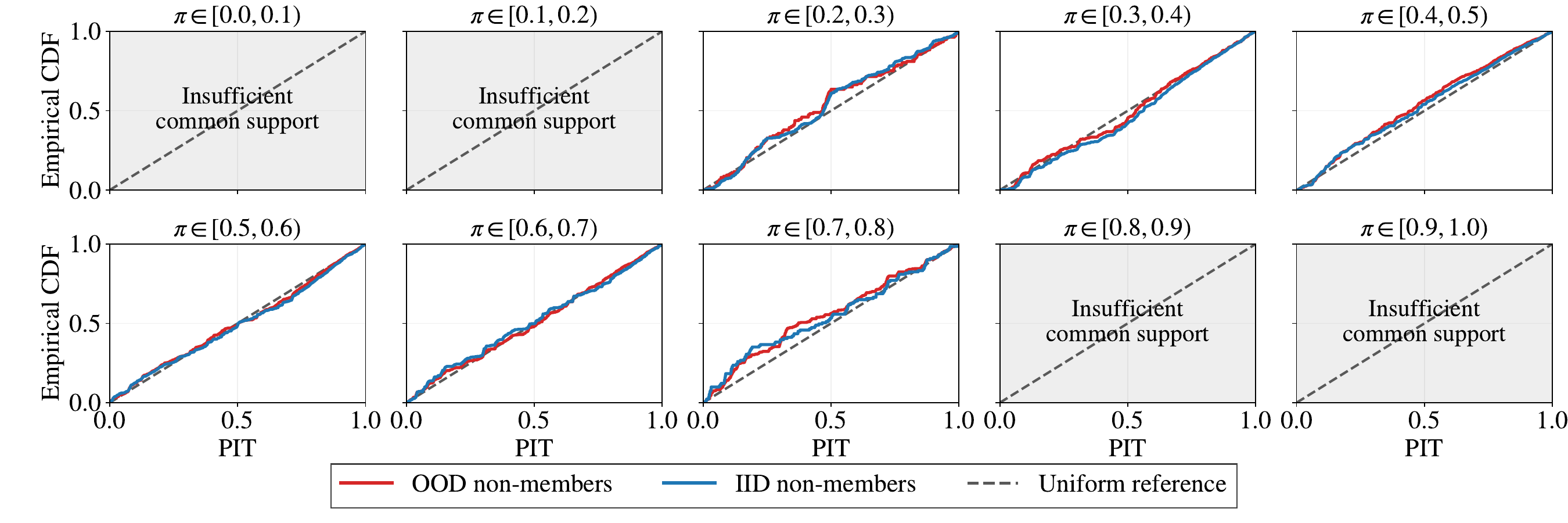}
                 \subcaption{CIFAR-10-C with severity level: 4}
            \end{minipage}
            \caption{\textbf{Conditional-invariance diagnostics for CIFAR-10-C (severities 1–4).} Empirical CDFs of PIT values for OOD (red) and IID (blue) non-members are compared within estimated domain-propensity bins. Agreement between the curves supports conditional invariance. Dashed lines indicate the uniform reference; gray panels indicate insufficient common support.}
            \label{fig:pit_cifar10c}
        \end{center}
    \end{figure}

 \begin{figure}[htbp] 
        \begin{center}
            \begin{minipage}[t]{0.9\textwidth}
                \centering
                \includegraphics[width=\textwidth]{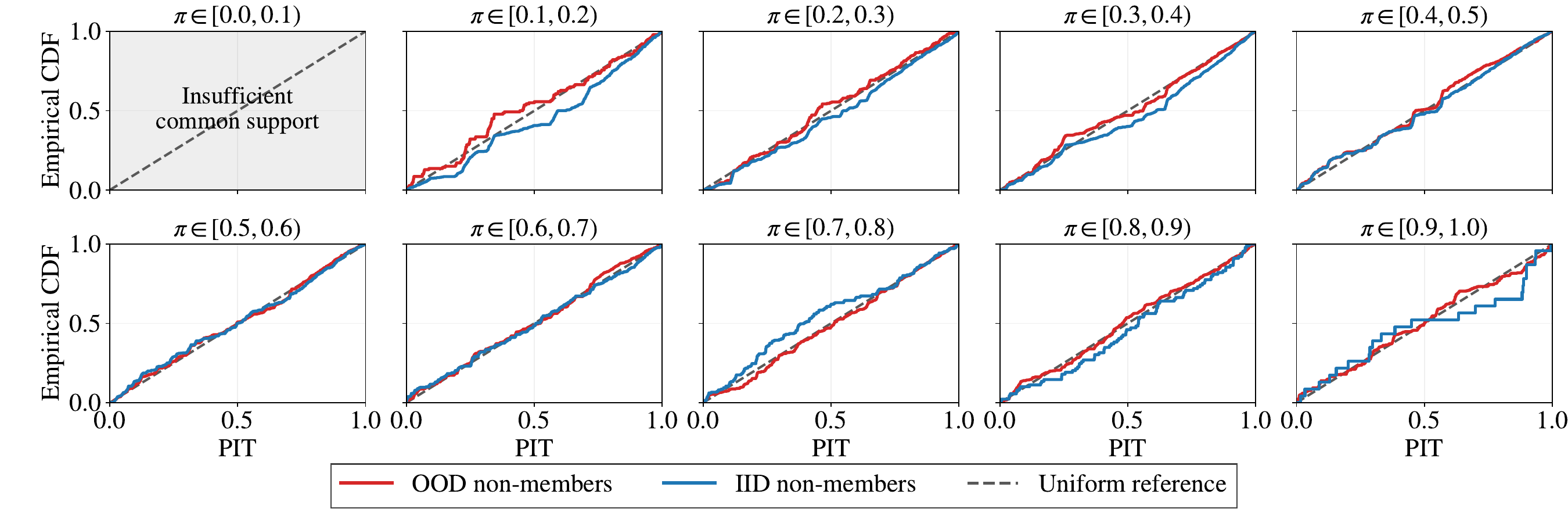}
                \subcaption{CIFAR-10-C with severity level: 5}
            \end{minipage}
            \begin{minipage}[t]{0.9\textwidth}
                \centering
                \includegraphics[width=\textwidth]{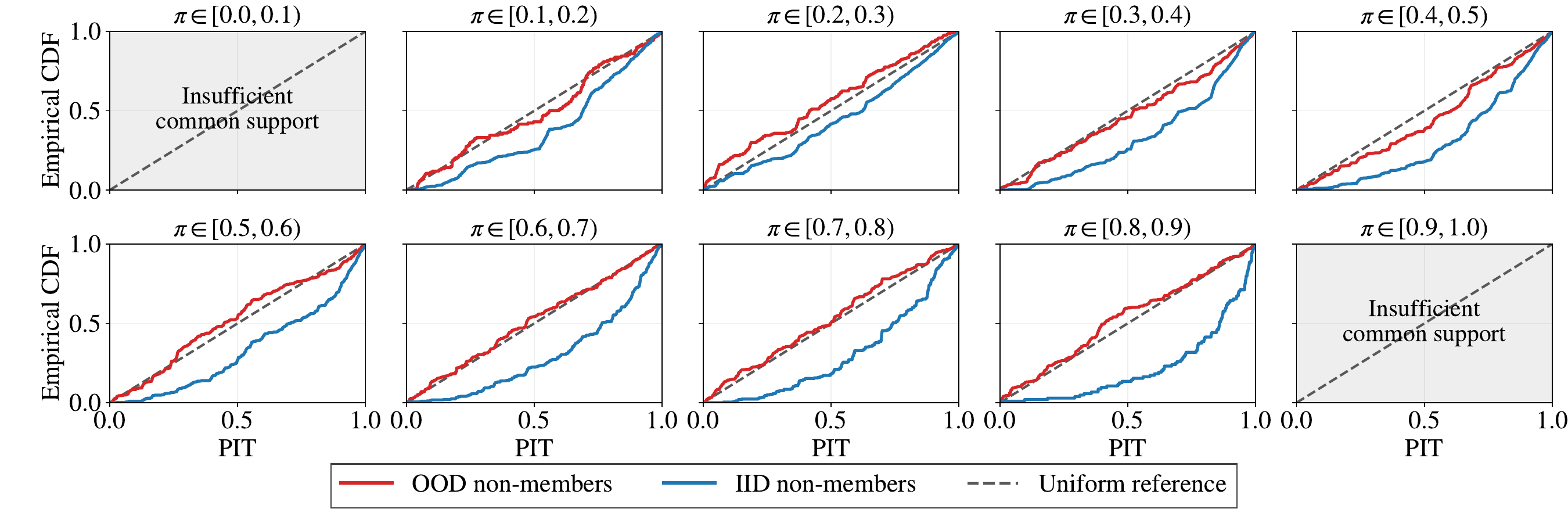}
                 \subcaption{CINIC-10 }
            \end{minipage}
             \begin{minipage}[t]{0.9\textwidth}
                \centering
                \includegraphics[width=\textwidth]{Figures/domainnet.pdf}
                 \subcaption{DomainNet} 
            \end{minipage}
             \begin{minipage}[t]{0.9\textwidth}
                \centering
                \includegraphics[width=\textwidth]{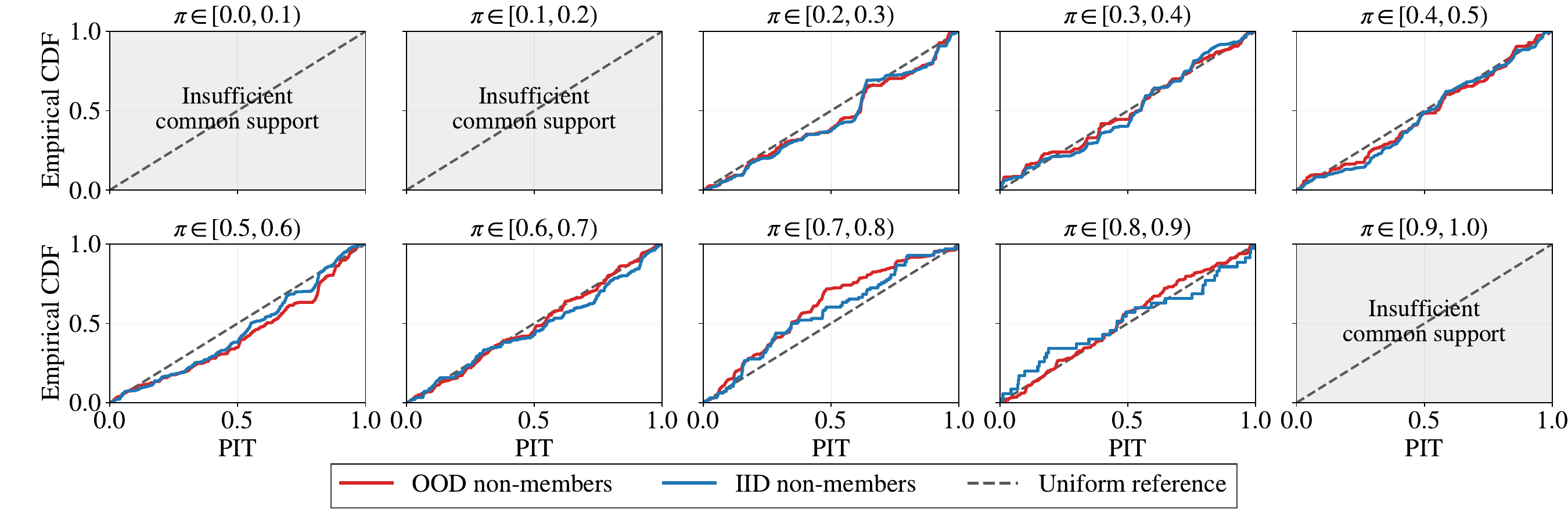}
                 \subcaption{arXiv}
            \end{minipage}
            \caption{\textbf{Conditional-invariance diagnostics for additional distribution shifts.} PIT empirical CDFs for OOD ({red}) and IID ({blue}) non-members across CIFAR-10-C severity 5, CINIC-10, DomainNet, and arXiv. Curves are compared within estimated domain-propensity bins; gray panels indicate insufficient common support. Agreement between the two curves indicates greater consistency with conditional invariance, while IID non-member curves below the diagonal indicate larger MIA scores than predicted by the OOD reference distributions.}
            \label{fig:pit_additional}
        \end{center}
    \end{figure}

To characterize the degree of distribution shift and the available overlap between IID and OOD non-members, we examine the estimated domain propensity score \(\hat{\pi}(Z)=\Pr(D=\mathrm{OOD}\mid Z,\text{non-member})\).
Figure~\ref{fig:domain_propensity} shows its distribution for the two domains across all datasets. 
When the two distributions substantially overlap, IID and OOD samples share regions of representation space in which conditional comparisons are better supported. 
In contrast, strong separation indicates that the domains are easily distinguishable and that conditional transport must rely more heavily on extrapolation. 
As expected, the separation generally increases with corruption severity for CIFAR-10-C, while datasets such as DomainNet exhibit considerably weaker overlap.

\begin{figure}[htbp] 
    \begin{center}          
    \centering
        \includegraphics[width=\textwidth]{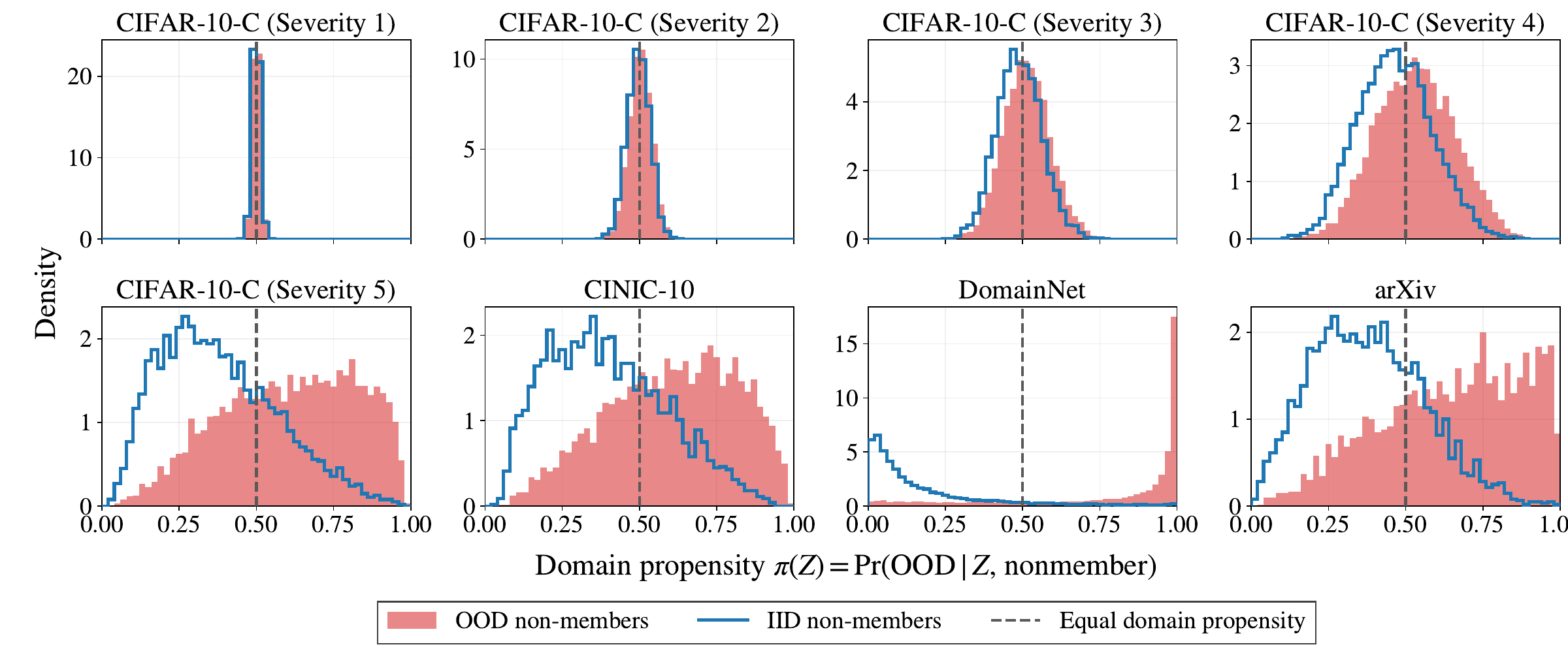}     
    \caption{\textbf{Domain-propensity distributions across datasets. }Distributions of \(\hat{\pi}(Z)=\Pr(D=\mathrm{OOD}\mid Z,\text{non-member})\) for IID (blue) and OOD (red) non-members. The dashed line marks equal domain propensity, \(\hat{\pi}(Z)=0.5\). Separation between the distributions reflects domain distinguishability, while their overlap indicates common support.}
    \label{fig:domain_propensity}
    \end{center}
    \vskip -0.1in
\end{figure}

\section{Recovering a conservative transported score distribution} \label{app:recover_scores}
Our method operates directly at the level of the privacy estimate: given the privacy lower bound \(\hat\mu_{\mathrm{trans}}\) obtained from the nominal transported non-member distribution and the calibrated error estimate \(\hat r\), we report
\begin{equation}
    \hat\mu_{\mathrm{corr}} = \max(0,\hat\mu_{\mathrm{trans}}-\hat r). 
\end{equation}

For applications where a corrected non-member score distribution is nevertheless useful, we construct one whose induced privacy estimate is consistent with this scalar correction. Intuitively, we conservatively shrink the nominal transported distribution toward the null distribution, represented by the empirical member-score distribution. When the non-member distribution coincides with the member distribution, the two populations are indistinguishable from their scores and the corresponding privacy lower bound is zero.

Let \(\tilde F_{S_N^P}\) denote the nominal transported non-member score distribution and \(F_{S_M}\) the empirical member-score distribution. We interpolate between them in quantile space:
\begin{equation}
     F^{-1}_{\gamma}(u) = (1-\gamma)\tilde F_{S_N^P}^{-1}(u) + \gamma F_{S_M}^{-1}(u), \qquad u\in(0,1),
\end{equation}
evaluated on the midpoint grid \(u_j=(j-\tfrac12)/n\). Here, \(\gamma\in[0,1]\) controls the amount of conservative retreat: \(\gamma=0\) retains the nominal transported distribution, whereas \(\gamma=1\) recovers the member distribution.
We then choose the smallest amount of retreat whose induced privacy lower bound does not exceed the calibrated target:
\begin{equation}
    \gamma^* = \min \left\{ \gamma\in\mathcal G: \hat\mu(F_\gamma;F_{S_M}) \le \hat\mu_{\mathrm{corr}} \right\}, \qquad \mathcal G= \left\{ 0,\frac1G,\frac2G,\ldots,1 \right\}.
\end{equation}
The resulting \(F_{\gamma^*}\) therefore provides a distribution-level realization of our conservative scalar correction. Because \(\gamma\) is searched over a finite grid, its induced privacy estimate need not equal \(\hat\mu_{\mathrm{corr}}\) exactly, but is constrained not to exceed it.

\begin{lemma}
    Fix all member scores \(\mathcal S_M\), all non-member scores \(\mathcal S_N^Q\) except for \(s_i\), and fixed tail sizes \(k_-\) and \(k_+\). Assume scores are distinct (or ties are resolved by a fixed rule). If record \(i\) is a non-member, then the number of correct guesses \(k(\mathcal S_N^Q, \mathcal S_M)\) is non-increasing in \(s_i\).
\end{lemma}
\begin{proof}
As \(s_i\) increases, its rank can change only when it crosses another score \(s_j\). 
Crossings that do not change membership of the bottom \(k_-\) or top \(k_+\) tails leave \(k\) unchanged. 
Thus only two cases matter.
\begin{enumerate}[leftmargin=2em]
    \item \(i\) leaves the bottom tail and \(j\) enters it. 
    Before the crossing, \(i\) is correctly guessed as a non-member and contributes \(1\).
    After the crossing, \(j\) contributes \(1\) only if \(j\) is also a non-member. 
    Hence
    \begin{equation}
        \Delta k = -1+\mathbf 1\{j\text{ non-member}\} \le 0. 
    \end{equation}
    \item \(i\) enters the top tail and displaces \(j\). 
    Since \(i\) is a non-member, the member guess is incorrect, so \(i\) contributes \(0\). 
    The displaced record \(j\) loses a correct guess only when \(j\) is a member. 
    Hence
    \begin{equation}
        \Delta k = -\mathbf 1\{j\text{ member}\} \le 0.
    \end{equation}
\end{enumerate}

\end{proof}

\begin{corollary}
    If \(s_{i,N}^Q\le s_{i,N}^{Q'}\) for every \(i\), then \(\hat\mu(\mathcal S_N^{Q'},\mathcal S_M)\le\hat\mu(\mathcal S_N^{Q},\mathcal S_M).\)
\end{corollary}
This follows because \(\mu\) is non-decreasing in \(k\).
Raising non-member scores never raises the bound.

\begin{figure}[htbp] 
    \begin{center}          
    \centering
        \includegraphics[width=\textwidth]{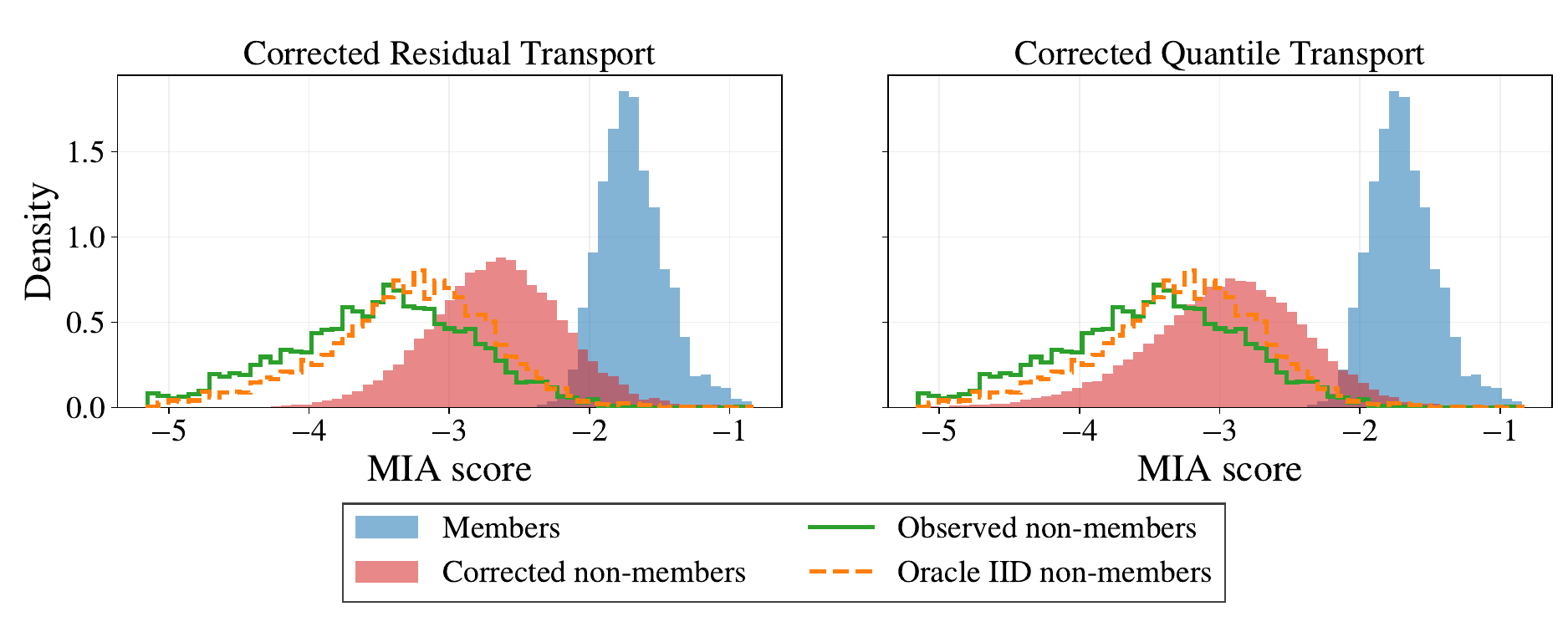}     
    \caption{\textbf{Distribution-level realization of the conservative correction.}  Results are shown for residual transport (left) and quantile transport (right).}
    \label{fig:transported_scores}
    \end{center}
    \vskip -0.1in
\end{figure}

Figure~\ref{fig:transported_scores} illustrates this optional distribution-level correction for both residual and quantile transport on arXiv data. 
The nominal transported non-member distribution is retreated toward the member distribution until its induced privacy lower bound reaches the calibrated conservative target. 
This visualization shows how the scalar correction used in the main method can be translated into a conservative correction of the transported score distribution.

Figure~\ref{fig:transported_scores} also helps explain the relative performance of our two transport variants. 
On arXiv, the shift between the OOD and IID non-member MIA score distributions is largely a change in location, while their variances remain similar. Consequently, both residual and quantile transport perform well. 
In contrast, DomainNet exhibits relatively limited support overlap, as shown in \Cref{fig:domain_propensity}. 
The quantile model must therefore estimate conditional score distributions in feature regions that are poorly represented in the source data.
Because conditional tail estimation is particularly sensitive to extrapolation and becomes unreliable when source support is sparse, corrected quantile transport performs worse than residual transport in this setting \citep{liang2024blessings}.




\end{document}